\documentclass[10pt,journal,compsoc]{IEEEtran}
\usepackage{hyperref}       
\usepackage{url}            
\usepackage{booktabs}       
\usepackage{amsfonts}       
\usepackage{nicefrac}       
\usepackage{microtype}      
\usepackage{enumerate}
\usepackage{graphicx}
\usepackage{subfigure}
\usepackage{booktabs}
\usepackage{multirow}
\usepackage{amsmath}
\usepackage{amssymb}
\usepackage{graphicx}
\usepackage{subfigure}
\usepackage{wrapfig}
\usepackage{enumitem}
\usepackage{hyperref}
\usepackage{pifont}
\usepackage{color}
\usepackage[ruled]{algorithm2e}
\usepackage{algorithmic}
\usepackage{amsmath}
\usepackage{amsthm}
\usepackage{bbm}
\usepackage{caption}
\usepackage[switch]{lineno}
\newcommand{\confmix}{\textsc{RMix}\xspace}
\newcommand{\Umix}{\textsc{RMix}\xspace}
\newcommand{\IWmix}{\textsc{RMix}\xspace}

\newtheorem{assumption}{Assumption}[section]

\newtheorem{theorem}{Theorem}[section]
\newtheorem{lemma}[theorem]{Lemma}
\newtheorem{prop}[theorem]{Proposition}

\usepackage{setspace}
\newenvironment{sproof}
 {\vspace{-8pt}\begin{spacing}{0.3}\begin{proof}}
 {\end{proof}\end{spacing}\vspace{-3pt}}
\usepackage{ragged2e}
\usepackage{floatrow}
\newfloatcommand{capbtabbox}{table}[][\FBwidth]
\usepackage{tikz}

\usepackage{array}
\newcolumntype{L}[1]{>{\raggedright\let\newline\\\arraybackslash\hspace{0pt}}m{#1}}
\newcommand{\E}{\mathbb{E}}
\newcommand{\R}{\mathbb{R}}
\usepackage{mathtools}
\usepackage{paralist}
\let\itemize\compactitem
\let\enditemize\endcompactitem
\let\enumerate\compactenum
\let\endenumerate\endcompactenum

\ifCLASSOPTIONcompsoc
  \usepackage[nocompress]{cite}
\else
  \usepackage{cite}
\fi
\ifCLASSINFOpdf
\else
\fi

\begin{document}
\title{
Reweighted Mixup for Subpopulation Shift
}
\author{Zongbo Han, Zhipeng Liang, Fan Yang, Liu Liu, Lanqing Li, Yatao Bian, Peilin Zhao, \\  Qinghua Hu,  Bingzhe Wu, Changqing Zhang, Jianhua Yao
\IEEEcompsocitemizethanks{
\IEEEcompsocthanksitem Z. Han, C. Zhang and Q. Hu are with the College of Intelligence and Computing, Tianjin University, Tianjin 300072, China (e-mail: \{zongbo, zhangchangqing, huqinghua\}@tju.edu.cn). \\
\IEEEcompsocthanksitem Z. Liang is with Department of Industrial Engineering and Decision Analytics, Hong Kong University of Science and Technology, Hongkong 999077, China (e-mail: zliangao@connect.ust.hk). \\
\IEEEcompsocthanksitem F. Yang, Liu Liu, B. Wu, J. Yao, La. Li, Y. Bian and P. Zhao are with Tencent AI Lab, Shenzhen 518054, China (e-mail: \{fionafyang, leonliuliu, bingzhewu, jianhuayao\}@tencent.com, \{lanqingli1993, yatao.bian\}@gmail.com, and peilinzhao@hotmail.com). \\
\IEEEcompsocthanksitem This work is partially done by Zongbo Han and Zhipeng Liang during their internship at Tencent AI Lab. Corresponding authors: Bingzhe Wu,   Yao Jianhua and Changqing Zhang.
}
}
\markboth{In submission}%
{Shell \MakeLowercase{\textit{et al.}}: Bare Demo of IEEEtran.cls for Computer Society Journals}

\IEEEtitleabstractindextext{%
\begin{abstract}
\justifying
\textcolor{black}{
Subpopulation shift exists widely in many real-world applications, which refers to the training and test distributions that contain the same subpopulation groups but with different subpopulation proportions. Ignoring subpopulation shifts may lead to significant performance degradation and fairness concerns. Importance reweighting is a classical and effective way to handle the subpopulation shift. However, recent studies have recognized that most of these approaches fail to improve the performance especially when applied to over-parameterized neural networks which are capable of fitting any training samples. In this work, we propose a simple yet practical framework, called reweighted mixup (\IWmix), to mitigate the overfitting issue in over-parameterized models by conducting importance weighting on the ``mixed'' samples. Benefiting from leveraging reweighting in mixup, \IWmix allows the model to explore the vicinal space of minority samples more, thereby obtaining more robust model against subpopulation shift. When the subpopulation memberships are unknown, the training-trajectories-based uncertainty estimation is equipped in the proposed \IWmix to flexibly characterize the subpopulation distribution. We also provide insightful theoretical analysis to verify that \IWmix achieves better generalization bounds over prior works. Further, we conduct extensive empirical studies across a wide range of tasks to validate the effectiveness of the proposed method.}
\end{abstract}
}

\maketitle
\IEEEdisplaynontitleabstractindextext
\IEEEpeerreviewmaketitle

\IEEEraisesectionheading{\section{Introduction}\label{sec:introduction}}
\IEEEPARstart{D}{}\textcolor{black}{istribution shift is a common challenge for empirical risk minimization~(ERM), which refers to that the data used for training and test have different distributions \cite{shimodaira2000improving, huang2006correcting, arjovsky2019invariant}. A specific form of distribution shift is subpopulation shift wherein the training and test distributions consist of the same subpopulation groups but differ in subpopulation frequencies \cite{barocas2016big,bickel2007discriminative}. Many practical research problems (e.g., fairness of machine learning and class imbalance) can all be considered as a special case of subpopulation shift \cite{mincu2022developing,koenecke2020racial,hashimoto2018fairness,japkowicz2000class}. For example, when we conduct training on a dataset with biased demographic subpopulations, the obtained model may be unfair to the minority subpopulations \cite{koenecke2020racial,mincu2022developing,hashimoto2018fairness}.  Therefore the essential goal of fair machine learning is to mitigate the performance gap between different subpopulations in the test datasets.}

\textcolor{black}{Plenty of approaches have been proposed for solving the subpopulation shift problem. Among these approaches, importance weighting (IW) is a classical yet effective technique by imposing static or dynamic weight on each sample when building weighted empirical loss. Therefore each subpopulation group contributes comparably to the final training objective. Specifically, there are normally two ways to achieve importance reweighting. Early works propose to reweight the sample inverse proportionally to the subpopulation frequencies (i.e., static weights) \cite{shimodaira2000improving,sagawa2020investigation,cui2019class,Sagawa2020Distributionally,cao2019learning,liu2011class}, such as class-imbalanced learning approaches \cite{cui2019class,cao2019learning,liu2011class}. 
Alternatively, a more flexible way is to assign sample-wise importance weights to individual samples adaptively according to training dynamics \cite{wen2014robust,zhai2021boosted,michel2021modeling,zhai2021doro,lahoti2020fairness, michel2022distributionally,lin2017focal,shu2019meta}. Distributional robust optimization (DRO) is one of the most representative methods in this line, which minimizes loss over the worst-case distribution in a neighborhood of the empirical training distribution. A commonly used dual form of DRO can be regarded as a special case of importance reweighting wherein the importance weights of training samples are updated based on the current loss \cite{namkoong2016stochastic,hu2018does,levy2020large,hu2013kullback} in an alternated manner.}

\textcolor{black}{However, some recent studies have shown both empirically and theoretically that these IW methods could fail to achieve better worst-case subpopulation performance compared with ERM. Empirically, prior works \cite{byrd2019effect,Sagawa2020Distributionally} recognize that various IW methods tend to exacerbate overfitting,  which leads to a diminishing effect on stochastic gradient descent over training epochs especially when they are applied to over-parameterized neural networks. Theoretically, previous studies prove that for over-parameterized neural networks, reweighting algorithms do not improve over ERM because their implicit biases are (almost) equivalent \cite{zhai2022understanding,sagawa2020investigation,xu2021understanding}. In addition, prior work also points out that using conventional regularization techniques such as weight decay cannot significantly improve the performance of IW \cite{Sagawa2020Distributionally}. Therefore, in this paper, we propose to use more exploratory importance reweighting equipped with mixup strategies to eliminate model overfitting.}

\textcolor{black}{To this end, we introduce a novel technique called reweighted mixup (\IWmix), by reweighting the mixed samples according to the subpopulation memberships within the mini-batch while mitigating overfitting. Specifically, we employ the well-known mixup technique to produce ``mixed'' augmented samples. Then we train the model on these mixed samples to make sure it can always see ``novel'' samples thus the effects of IW will not dissipate even at the end of the training epoch.
To enforce the model to perform fairly well on all subpopulations, we further efficiently reweight the mixed samples according to the subpopulation memberships of the original samples. When the subpopulation memberships are unknown, uncertainty-based importance weights are proposed by assigning importance weights according to the uncertainty of training samples to better explore the minority subpopulation samples.
The reweighted mixup loss function is induced by combining the weighted losses of the corresponding two original samples. At a high level, this approach augments training samples in an importance-aware manner, i.e., putting more focus on samples that belong to minority subpopulations when constructing ``mixed'' samples and corresponding losses. We also show \IWmix can provide an additional theoretical benefit that achieves a tighter generalization bound than weighted ERM
\cite{liu2021just,lin2017focal,zhai2021doro,levy2020large}.}

\textcolor{black}{Compared with the conference version \cite{han2022umix}, we significantly improve our previous work in the following aspects: (1) further reweighting analysis of the proposed \IWmix by providing a more intuitive comparison with previous mixup-based methods (Sec.~\ref{sec:idanalysis}); (2) extension of the proposed method in the group-aware setting which enables the model to use subpopulation memberships information (Sec.~\ref{sec:group-aware-iw}); (3) more theoretical extension to demonstrate that the proposed method is compatible with both vanilla mixup \cite{zhang2018mixup} and \textcolor{black}{CutMix-based mixup\cite{yun2019cutmix}, and can achieve a tighter generalization error bound (Sec.~\ref{sec:theory})}; (4) more experiments and discussions to verify the effectiveness of the proposed method in both group-aware and group-oblivious settings (Sec.~\ref{sec:experiments}). Overall, the proposed \IWmix supported by theoretical guarantees significantly improves previous methods against subpopulation shift. The contributions of this paper are summarized as follows:
\begin{itemize}
\item [(1)] We propose a simple and practical approach called reweighted mixup (\confmix) to improve previous IW methods by reweighting the mixed samples, which provides a new framework to mitigate overfitting in over-parameterized neural networks.
\item [(2)] Under the proposed framework, we provide two strategies to assign the importance weights making the proposed method work well in both group-oblivious and group-aware settings. Especially for group-oblivious setting, we propose an uncertainty-based importance weighting strategy, which could accurately characterize the subpopulation in training. 
\item [(3)] We provide an insightful theoretical analysis that \confmix can achieve a tighter generalization bound than previous importance weighting methods. Moreover, a qualitative analysis with strong intuition is provided to illustrate the reweighting effect of the proposed method over the previous mixup.
\item [(4)] We conduct extensive experiments on a wide range of tasks, and the proposed \confmix achieves state-of-the-art performance in both group-oblivious and group-aware settings.
\end{itemize}}

\section{Related Work\label{sec:relate}}
\subsection{Importance weighting for subpopulation shift}
To improve the model robustness against subpopulation shift, importance weighting (IW) is a classical yet effective technique by imposing static or adaptive weight on each sample and then building weighted empirical loss. Therefore each subpopulation group can have a comparable strength in the final training objective. Specifically, there are typically two ways to achieve importance reweighting, i.e., using static or adaptive importance weights.

\textbf{Static methods}. The naive reweighting approaches perform static reweighting based on the distribution of training samples \cite{shimodaira2000improving,sagawa2020investigation,cui2019class,Sagawa2020Distributionally,cao2019learning,liu2011class}. Their core motivation is to make different subpopulations have a comparable contribution to the training objective by reweighting. Specifically, the most intuitive way is to set the weight of each sample to be inversely proportional to the number of samples in each subpopulation \cite{shimodaira2000improving, sagawa2020investigation,Sagawa2020Distributionally}. Besides, there are some methods to obtain sample weights based on the effective number of samples \cite{cui2019class}, subpopulation margins \cite{cao2019learning}, and Bayesian networks \cite{liu2011class}. 

\textbf{Dynamic methods}. 
In contrast to the above static methods, a more essential way is to assign each individual sample an adaptive weight that can vary according to training dynamics \cite{wen2014robust,zhai2021boosted,michel2021modeling,zhai2021doro,lahoti2020fairness, michel2022distributionally,lin2017focal,shu2019meta}. Distributional robust optimization (DRO) is one of the most representative methods in this line, which minimizes the loss over the worst-case distribution in a neighborhood of the empirical training distribution. A commonly-used dual form of DRO can be considered as a special case of importance reweighting wherein the importance weight is updated based on the current loss \cite{namkoong2016stochastic,hu2018does,levy2020large,hu2013kullback} in an alternated manner. For example, in the group-aware setting (i.e., we know each sample belongs to which subpopulation), GroupDRO \cite{Sagawa2020Distributionally} introduces an online optimization algorithm to update the weights of each group. In the group-oblivious setting, existing methods mostly \cite{wen2014robust,lahoti2020fairness,michel2021modeling, michel2022distributionally} model the problem as a (regularized) minimax game, where one player aims to minimize the loss by optimizing the model parameters and another player aims to maximize the loss by assigning weights to each sample.

\subsection{Mixup-based methods}
Mixup training \cite{zhang2018mixup} has been both theoretically \cite{zhang2021how,JMLR:v23:20-1385} and experimentally \cite{zhang2018mixup, yun2019cutmix, verma2019manifold} shown that it could improve the generalization and robustness of the model. Specifically, mixup could extend the training distribution by conducting a convex combination between any pair of training samples, which is equivalent to minimizing the vicinal risk of the training distribution \cite{chapelle2000vicinal}.

Many mixup variants are proposed to improve the vanilla mixup by changing how the training samples are combined. Typically, CutMix \cite{yun2019cutmix} cuts and pastes the original training image to avoid the blurring and noise of the generated images caused by vanilla mixup using linear interpolation. Puzzlemix \cite{kim2020puzzle} further improves CutMix by selecting cut and pasted image patches via saliency information, thus preserving key information of samples before mixing. Adversarial mixup resynthesis \cite{beckham2019adversarial} introduces adversarial learning to train a mixing function that could generate more realistic mixed samples. Manifold mixup \cite{verma2019manifold} obtains neural networks with smoother decision boundaries by mixing samples in representation space. To explore the representation space more, patch up \cite{faramarzi2020patchup} mixes and exchanges blocks of any training samples in the feature map space generated by the convolutional neural network. 

Mixup has also been extensively studied to improve the generalization of the model \cite{zhou2020domain, yao2022improving, xu2020adversarial,yao2022cmix, hong2021stylemix}. Typically, 
existing methods \cite{hong2021stylemix, xu2020adversarial,hong2021stylemix} try to mix the samples to synthesize novel domains, thereby improving the generalization. Besides, by selecting mixed sample pairs during training, LISA \cite{yao2022improving} and cmix \cite{yao2022cmix} improve the generalization of the model in classification and regression respectively. 
Different from the above methods, we do not focus on how to mixup the training samples, but improve the subpopulation shift robustness of the model by re-weighting the loss of the mixed samples. Besides, our work is orthogonal to the previous methods, i.e., we can use our weight-building strategy to improve the performance of previous mixup-based methods.

\subsection{Uncertainty quantification}
When the subpopulation memberships are unknown, the core of our method is based on the high-quality uncertainty quantification of each sample. There are many approaches proposed to quantify the uncertainty. Specifically, the uncertainty of deep learning models includes epistemic (model) uncertainty and aleatoric (data) uncertainty \cite{kendall2017uncertainties}. To obtain the epistemic uncertainty, Bayesian neural networks (BNNs) \cite{neal2012bayesian,mackay1992practical,denker1990transforming,kendall2017uncertainties} have been proposed which replace the deterministic weight parameters of the model with the distribution. Unlike BNNs, ensemble-based methods obtain the epistemic uncertainty by training multiple models and ensembling them \cite{lakshminarayanan2017simple,havasi2021training,antoran2020depth,huang2017snapshotensembles}. Aleatoric uncertainty focuses on the inherent noise in the data, which usually is learned as a function of the data \cite{kendall2017uncertainties,le2005heteroscedastic,nix1994estimating}. Uncertainty quantification has been successfully equipped in many fields such as multimodal learning \cite{ma2021trustworthy,han2022trusted,geng2021uncertainty}, multitask learning \cite{kendall2018multi,deng2021iterative}, and reinforcement learning \cite{kalweit2017uncertainty,li2021mural}.
Unlike previous methods, our method focuses on estimating the epistemic uncertainty of training samples with subpopulation shifts and then reweighting uncertain samples, thereby improving the performance of minority subpopulations with high uncertainty.


\section{Preliminary}
\textcolor{black}{In this section, the necessary background and notations are clarified. Let the input and label space be $\mathcal{X}$ and $\mathcal{Y}$ respectively. Given a training dataset $\mathcal{D}$ with $N$ training samples $\{(x_i, y_i)\}_{i=1}^{N}$ i.i.d. sampled from a probability distribution $P$, we consider the setting that the training distribution $P$ is a mixture of $G$ predefined subpopulations (also known as groups), i.e., $P=\sum_{g=1}^{G} k_g P_g$, where $k_g$ and $P_g$ denote the $g$-th subpopulation’s proportion and distribution respectively ($\sum_{g=1}^Gk_g=1$). Our goal is to obtain a model $f_{\theta}: \mathcal{X}\rightarrow \mathcal{Y}$ parameterized by $\theta \in \Theta$ that performs well on all subpopulations.}

\textcolor{black}{The well-known empirical risk minimization (ERM) algorithm doesn't consider the subpopulations and minimizes the following expected loss
\begin{equation}
\textcolor{black}{\mathbb{E}_{(x, y) \sim P}[\ell(\theta, x, y)],}
\end{equation}
where $\ell: \Theta \times \mathcal{X} \times \mathcal{Y}\rightarrow \mathbb{R}_{+}$ denotes the loss function. For training data contains multiple subpopulations, the expected loss of ERM can be written as
\begin{equation}
\sum_{g=1}^{G} k_g\mathbb{E}_{(x, y) \sim P_{g}}[\ell(\theta, x, y)].    
\end{equation}
This may lead to the model paying more attention to the majority subpopulations (i.e., subpopulations with a larger proportion) in the training set and resulting in poor performance on the minority subpopulations (i.e., subpopulations with a smaller proportion). For example, the ERM-based models may learn spurious correlation between the majority subpopulations and labels but this spurious correlation does not hold in the minority subpopulations\cite{Sagawa2020Distributionally}.}

\textcolor{black}{In this paper, we focus on learning a model that is \textcolor{black}{robust} against subpopulation shift by importance reweighting, i.e., we can still achieve better performance even on the worst-case subpopulation among all the subpopulations. 
Specifically, importance reweighting tries to assign an importance weight to each sample, therefore each subpopulation could contribute comparably to the final loss, e.g.,
\begin{equation}
\label{eq:iwparadigm}
\mathbb{E}_{(x, y) \sim P}[w(x,y)\ell(\theta, x, y)],
\end{equation}
where the function $w(x,y)$ is the weighted function to return the weight of the sample $(x,y)$. This article hopes to improve the importance reweighting
algorithm by equipping it with mixup.}


\textcolor{black}{\textbf{Setting of subpopulation shift.} Previous works on improving subpopulation shift robustness investigate \textcolor{black}{two} different settings, i.e., \textcolor{black}{group-aware and group-oblivious} \cite{zhai2021doro,liu2021just,Sagawa2020Distributionally}. Most of the previous works assume that the group labels are available during training \cite{Sagawa2020Distributionally, yao2022improving}, i.e., which subpopulation the training sample belongs to is known. This is called group-aware setting. However, we may not have training group labels in many real applications \cite{hashimoto2018fairness}. Meanwhile, the group label information may not be available due to privacy concerns \cite{lahoti2020fairness,hashimoto2018fairness}.
This paper studies both group-aware and group-oblivious settings under a unified reweighting mixup framework. On the group-aware setting, we set weights for different subpopulations according to the group information. Besides, on the group-oblivious setting, we cannot obtain group information for each example at training time. This requires the model to identify underperforming samples and then pay more attention to them during training.}

\section{Method\label{sec:method}}

\textcolor{black}{This section presents the technical details of \IWmix. \textcolor{black}{The proposed method can promote the prediction fairness on all subpopulations by reweighting the mixed samples. Intuitively, by reweighting the mixed samples instead of the original training samples,} the deep neural networks especially overparameterized neural networks can always encounter “novel” samples, which avoids the model from memorizing the minority subpopulations rather than learning the intrinsic correlation between inputs and labels. From another perspective, \IWmix assigns greater weights to the minority subpopulation samples when creating mixed samples, so the model can better explore the vicinity space of the minority samples, thereby making the model perform well on all subpopulations uniformly. Last but not least, \IWmix is a general reweighting framework that can work well in both group-oblivious or group-aware settings. When the group label is known, we can simply assign a negative-correlated importance weight to each group based on their sample size. When the group label is unknown, we use uncertainty information to assign importance weights to training samples.}

\textcolor{black}{We first introduce the basic procedure of \IWmix in Sec.~\ref{sec:iwmix}, and then provide a deeper analysis of \IWmix to intuitively show the effects of importance weights in Sec.~\ref{sec:idanalysis}. Finally, we present how \IWmix assigns importance weights in both group-aware and group-oblivious settings in Sec.~\ref{sec:setweights}.
\subsection{Importance-weighted mixup}
\label{sec:iwmix}
It has been both theoretically and experimentally shown that the effect of importance reweighting could fail to achieve better worst-case subpopulation performance especially when over-parameterized neural networks are employed \cite{byrd2019effect,Sagawa2020Distributionally,zhai2022understanding,sagawa2020investigation,xu2021understanding}. \textcolor{black}{Specifically, deep neural networks tend to memorize minority samples with higher weights in the training set rather than learn the correlation between minority samples and labels, resulting in poor generalization on minority samples at test time.} To this end, \IWmix employs an aggressive data augmentation strategy called importance reweighted mixup to mitigate overfitting, which enables the model to see ``new'' samples during the training by reweighting the mixed samples, avoiding the model simply memorizing minority samples instead of learning for the intrinsical correlation between inputs and labels.}

\textcolor{black}{
Formally, vanilla mixup \cite{zhang2018mixup} constructs virtual training examples (i.e., mixed samples) by performing linear interpolations between data/features and corresponding labels
\begin{equation}
\label{eq:vamix}
\widetilde{x}_{i,j} = \lambda x_i + (1-\lambda) x_j, \; \widetilde{y}_{i,j} = \lambda y_i + (1-\lambda) y_j,
\end{equation}
where $(x_i, y_i), {(x_j, y_j)}$ are two samples drawn randomly from empirical training distribution and $\lambda \in [0,1]$ is usually sampled from a beta distribution $Beta(\alpha, \alpha)$. In computer vision tasks, another popular mixup variant is CutMix \cite{yun2019cutmix}, which constructs virtual training examples by performing cutting and pasting in the original samples, i.e., 
\begin{equation}
\label{eq:cutmix}
\widetilde{x}_{i,j} = M(\lambda) \odot x_i+ \left(1-M(\lambda)\right)\odot x_j, \;
\widetilde{y}_{i,j} = \lambda y_i + (1-\lambda) y_j,
\end{equation}
where $M(\lambda)$ is a binary mask randomly chosen covering $\lambda$ proportion of the input, and $\odot$ represents the element-wise product. Note that the underlying assumption of vanilla mixup and CutMix is the mixture of original inputs which leads to the same linear interpolation of the corresponding labels.}

\textcolor{black}{After constructing the virtual training examples, mixup optimizes the following loss function
\begin{equation}
\label{eq:mixuploss}
\mathbb{E}[\ell(\theta, \widetilde{x}_{i,j}, \widetilde{y}_{i,j})].
\end{equation}
When the cross entropy loss is employed, Eq.~\ref{eq:mixuploss} can be rewritten as
\begin{equation}
\label{eq:mixuploss_2}
\mathbb{E}[\lambda \ell(\theta, \widetilde{x}_{i,j}, y_i) + (1-\lambda)\ell(\theta, \widetilde{x}_{i,j}, y_j)].
\end{equation}
Eq.~\ref{eq:mixuploss_2} can be regarded as a linear combination (mixup) of $\ell(\theta, \widetilde{x}_{i,j}, y_i)$ and $\ell(\theta, \widetilde{x}_{i,j}, y_j)$\cite{zhang2018mixup}. Unfortunately, since  the previous mixup-based methods do not consider the subpopulations with poor performance, it has been shown experimentally to be non-robust against subpopulation shift \cite{yao2022improving}. To this end, we introduce a simple yet effective method called \IWmix, which further employs a reweighted linear combination of the original loss based on Eq.~\ref{eq:mixuploss_2} to encourage the learned model to pay more attention to subpopulations with poor performance.}

\textcolor{black}{In contrast to the previous importance reweighting paradigm in Eq.~\ref{eq:iwparadigm}, the importance weights of \IWmix are used on the mixed samples.
For the $i$-th sample $x_i$, we denote its importance weight as $w_i$. Once we obtain the importance weight of each sample, we can perform weighted linear combination of  $\ell(\theta, \widetilde{x}_{i,j}, y_i)$ and $\ell(\theta, \widetilde{x}_{i,j}, y_j)$ with the following rule
\begin{equation}
\label{eq:weighted-loss}
\mathbb{E}[\textcolor{purple}{{w}_i}\lambda \ell(\theta, \widetilde{x}_{i,j}, y_i) + \textcolor{purple}{w_j}(1-\lambda)\ell(\theta, \widetilde{x}_{i,j}, y_j)],
\end{equation}
where ${w}_i$ and ${w}_j$ denote the importance weight of the $i$-th and $j$-th samples respectively. $\widetilde{x}_{i,j}$ can be obtained with vanilla mixup (i.e., Eq.~\ref{eq:vamix}), CutMix (i.e., Eq.~\ref{eq:cutmix}) or other mixup variants. \textcolor{black}{Therefore, the proposed method can cooperate with various mixup variants to adapt to different downstream tasks.} In practice, to balance the \confmix and normal training, we set a hyperparameter $\sigma$ that denotes the probability to apply \confmix. The overall training pseudocode for \confmix is shown in Algorithm \ref{algo}.
\begin{algorithm}[!htbp]
\SetAlgoNoEnd 
    \caption{The training pseudocode of \confmix.\label{algo}\label{alg:UMIX}}
    \KwIn{
        Training dataset $\mathcal{D}$ and the corresponding importance weights $\mathbf{w}=[w_1,\cdots,w_N]$, hyperparameter $\sigma$ to control the probability of doing \confmix, and parameter $\alpha$\ of the beta distribution;
    }
    \For{each iteration}{
    Obtain training samples $(x_{i}, y_{i})$, $(x_{j}, y_{j})$ and the corresponding importance weights $w_i$, $w_j$\;
    Sample $p\sim$ Uniform(0,1)\;
    \textbf{if} $p<\sigma$ \textbf{then} \ Sample $\lambda \sim Beta(\alpha, \alpha)$; \textbf{else}\ $\lambda=0$\;
    Obtain the mixed input $\widetilde{x}_{i,j}$ with Eq.~\ref{eq:vamix} or Eq.~\ref{eq:cutmix}\;
    Obtain the loss of the model with $\textcolor{purple}{{w}_i}\lambda \ell(\theta, \widetilde{x}_{i,j}, y_i) + \textcolor{purple}{w_j}(1-\lambda)\ell(\theta, \widetilde{x}_{i,j}, y_j)$\;
    Update model parameters $\theta$ to minimize loss with an optimization algorithm.
    }
\end{algorithm}
}
\subsection{Rethinking reweighting effects in \IWmix}
\label{sec:idanalysis}
\textcolor{black}{In this section, the improvement of the proposed \IWmix over mixup is analyzed by further decomposing Eq.~\ref{eq:weighted-loss}, which can provide a stronger intuition and theoretical analysis of why the proposed method works. Compared with the mixup, the proposed \IWmix could achieve better worst-case subpopulation performance by exploring the vicinal space of the minority samples more.}  

\textcolor{black}{Formally, to analyze the effect of the importance reweighting, we convert Eq.~\ref{eq:weighted-loss} to the form of original mixup loss (i.e., Eq.~\ref{eq:mixuploss}). Then we have Proposition~\ref{prop:weighted}, which provides an intuitive comparison of \IWmix with original mixup \textcolor{black}{loss}.
\begin{prop}
\label{prop:weighted}
When the cross-entropy loss is employed, the loss of \IWmix (i.e., Eq.~\ref{eq:weighted-loss}) can be rewritten as
\begin{equation}
\begin{aligned}
\mathbb{E}[\overline{w}_{i,j} \ell(\theta, \widetilde{x}_{i,j}, \overline{y}_{i,j})],
\end{aligned}
\end{equation}
where
\begin{equation}
\begin{aligned}
\nonumber
\overline{w}_{i,j}=w_i\lambda+w_j(1-\lambda) \text{\ \ \ and},
\end{aligned}
\end{equation}
\begin{equation}
\begin{aligned}
\label{eq:weighted-label}
\overline{y}_{i,j}=\frac{w_i}{\overline{w}_{i,j}}\lambda y_i + \frac{w_j}{\overline{w}_{i,j}} (1-\lambda) y_j.
\end{aligned}
\end{equation}
$\overline{w}_{i,j}$ can be seen as a constant related with constants $\lambda, w_i$, and $w_j$. 
\end{prop}}
\textcolor{black}{
\begin{sproof}
(Proof of Proposition~\ref{prop:weighted}).
For $K$-class classification, given a well trained neural classifier $f_\theta: \mathcal{X}\rightarrow\mathcal{Y}$ that could produce the predicted probability of each class $f_\theta(x)=[f_{\theta}^1(x),\cdots, f_{\theta}^K(x)]$ and $y=[y^1, \cdots, y^K]$ denotes the one-hot vector form of label $y$. Then when the cross-entropy loss is employed, Eq.~\ref{eq:weighted-loss} can be written as
\begin{equation}
\begin{aligned}
\nonumber
&\mathbb{E}\{-\sum_{k=1}^{K}[w_i\lambda y_i^k+w_j(1-\lambda)y_j^k]\log f_\theta^k(\widetilde{x}_{i,j})\}\\
=&\mathbb{E}\{-\overline{w}_{i,j}\sum_{k=1}^K[\frac{w_i\lambda}{\overline{w}_{i,j}}y_i^k+\frac{w_j(1-\lambda)}{\overline{w}_{i,j}}y_j^k]\log f_\theta^k(\widetilde{x}_{i,j})\} \\
=&\mathbb{E}[-\overline{w}_{i,j} \sum_{k=1}^K \overline{y}_{i,j} \log f_\theta^k(\widetilde{x}_{i,j})] = \mathbb{E}[\overline{w}_{i,j} \ell(\theta, \widetilde{x}_{i,j}, \overline{y}_{i,j})].
\end{aligned}
\end{equation}
\end{sproof}
}

\begin{table}[!htbp]
    \renewcommand{\arraystretch}{1.3}
  \centering
  \caption{The improvements of \IWmix over mixup.}
    \begin{tabular}{ccc}
    \toprule
    \vspace{-0.2em}
          & \multicolumn{1}{c}{mixup} & \multicolumn{1}{c}{\IWmix} \\
    \toprule
    $y$  & $\lambda y_i + (1-\lambda) y_j$     & $\textcolor{purple}{\frac{w_i}{\overline{w}_{i,j}}}\lambda y_i + \textcolor{purple}{\frac{w_j}{\overline{w}_{i,j}}} (1-\lambda) y_j$ \\
    loss  & $\ell(\theta, \widetilde{x}_{i,j}, \widetilde{y}_{i,j})$  &  $\textcolor{purple}{\overline{w}_{i,j}} \ell(\theta, \widetilde{x}_{i,j}, \overline{y}_{i,j})$\\
    \bottomrule
    \end{tabular}%
  \label{tab:comparison_mixup}%
\end{table}

\textcolor{black}{Now based on the Proposition~\ref{prop:weighted}, we give a further comparison between \IWmix and the previous mixup where the results are shown in Tab.~\ref{tab:comparison_mixup}. 
Compared with previous mixup, the effect of importance reweighting in \IWmix is reflected in two aspects:}

\textcolor{black}{\textbf{(1) Reweighted label interpolation}. \IWmix introduces importance weights when constructing the labels of virtual samples, which allows more attention on the labels of minority samples. Specifically, given two training samples $(x_i,y_i), (x_j,y_j)$, unlike previous mixup, \IWmix conducts importance weighted label interpolation (i.e., Eq.~\ref{eq:weighted-label}). Then, when samples from the majority and minority subpopulations are mixed to construct virtual training examples, the labels of the minority subpopulations (usually with larger importance weights) could dominate the label mixing. Consequently, the labels of the minority subpopulations will be the main components of the mixed labels.}

\textcolor{black}{
\textbf{(2) Reweighted mixed loss}. Different from the loss of mixup, \IWmix pays more attention to the virtual samples mixed by minority subpopulations. Specifically, in the mixup loss, all mixed samples have the same importance weight while the samples from the minority subpopulations could have larger weights in the \IWmix loss. Thus the \textcolor{black}{minority subpopulations} contribute more to the loss function by increasing their weights.}

\textcolor{black}{All the above reweightings are conducted adpatively within the proposed framework by setting the weights in Eq.~\ref{eq:weighted-loss} and we will elaborate it in the following section.}

\subsection{Obtaining importance weights}
\label{sec:setweights}
\textcolor{black}{In this section, we present how to obtain the importance weights for training samples in Eq.~\ref{eq:weighted-loss} in the group-aware and group-oblivious settings, especially the group-oblivious setting which is more challenging due to the lack of group information for training samples. Specifically, we first introduce how to obtain the importance weights of each training sample according to their group information in the group-aware setting. Then, we mainly focus on how to obtain reasonable importance weights in the group-oblivious setting. In general, we propose an uncertainty estimation method to find samples with high uncertainty and then set importance weights according to the uncertainty of training samples.}

\textcolor{black}{
\subsubsection{Group-aware importance weights}
\label{sec:group-aware-iw}
In the group-aware setting, the key to improving the performance of the minority subpopulations is to reweight the training samples according to the group information of different samples, which means that we usually need to assign larger weights to minority subpopulations to promote their learning while giving relatively small weights to majority subpopulations. Since the proposed \IWmix is a general framework that can cooperate with any reweighting methods, to improve the efficiency and better verify the effectiveness of \IWmix, we employ a simple and effective way to determine the importance weights of different subpopulations.}

\textcolor{black}{In general, we leverage the sample size of each subpopulation to assign importance weights, where the sample size is closely related to the learning difficulty of different subpopulations. Specifically, similar to previous works \cite{piratla2022focus,Mahajan2018exploring,Mikolov2013DistributedRO}, instead of directly setting the weight of each sample to be inversely proportional to the sample size of each subpopulation, we empirically 
set the importance weight $w_i$ of each sample as a function of the square root of its training group size $\sqrt{k_gN}$, where $k_g$ and $N$ denote the $g$-th subpopulation’s proportion and training samples respectively. Formally, for the $i$-th sample belonging to subpopulation $g$, the corresponding importance weight $w_i$ is set as $w_i := \exp (C/\sqrt{k_gN})$, where $C\in \mathbb{R}_+$ is model capacity constant and can be seen as a hyperparameter. In fact, the scaling $\sqrt{k_gN}$ of different subpopulations could reflect how minority subpopulations are more prone to overfitting than majority groups due to the general size dependence of model-complexity-based generalization bounds \cite{Sagawa2020Distributionally,cao2019learning}, which can be intuitively understood as how easily a subpopulation can be learned by the neural network, i.e., the more samples in the subpopulation, the lower the learning difficulty. Therefore, setting the importance weights as a function of $\sqrt{k_gN}$ could better balance the contribution of different subpopulations during training.}

\subsubsection{Group-oblivious importance weights}
In the group-oblivious setting, finding samples with high uncertainty during training is the key to assigning importance weights. For instance, in DRO-based algorithms, the uncertainty set is constructed with the current loss \cite{namkoong2016stochastic,hu2018does,levy2020large,hu2013kullback}. However, the uncertain samples found in this way may constantly change during training, because the loss is affected by the reweighting effect and cannot accurately identify minority samples\cite{liu2021just}, resulting in these methods not always upweighting the minority subpopulations. To this end, a sampling-based stable uncertainty estimation method is introduced to better characterize the subpopulation information of each sample.

\textcolor{black}{Formally, given a well-trained neural classifier $f_{\theta}:\mathcal{X}\rightarrow\mathcal{Y}$ producing the predicted class $\hat{f}_{\theta}(x)$, a simple way to obtain the uncertainty $u$ of a sample $x$ is whether the sample is correctly classified. However, as pointed out in previous work \cite{lakshminarayanan2017simple}, a single model cannot accurately characterize the sampling uncertainty. Therefore, we propose to obtain the uncertainty through Bayesian sampling from the model posterior distribution $p(\theta; \mathcal{D})$. Specifically, given a sample $(x_i, y_i)$, we define the training uncertainty $u_i$ as
\begin{equation}
\label{eq:uncertainty}
u_i = \int \kappa(y_i, \hat{f_{\theta}}(x_i))p(\theta;\mathcal{D})d\theta, 
\end{equation}
where
\begin{equation}
\kappa(y_i, \hat{f}_\theta(x_i))= \begin{cases}0, & \text { if } y_i = \hat{f}_{\theta}(x_i)\\ 1, & \text { if }  y_i\neq \hat{f}_{\theta}(x_i)\end{cases}.
\end{equation}
Then, we can obtain an approximation of Eq.~\ref{eq:uncertainty} with $T$ Monte Carlo samples as $u_i \approx \frac{1}{T} \sum_{t=1}^T \kappa(y_i, \hat{f}_{\theta_t}(x_i))$.}

\textcolor{black}{In practice, sampling $\{\theta_t\}_{t=1}^{T}$ from the posterior (i.e., $\theta_t\sim p(\theta; \mathcal{D})$) is computationally expensive and sometimes even intractable since multiple training models need to be built \textcolor{black}{\cite{lakshminarayanan2017simple}} or extra approximation errors need to be introduced \textcolor{black}{\cite{ritter2018scalable}}. 
We propose to employ the information from the historical training trajectory to approximate the sampling process. More specifically, we train a model with simple ERM and record the prediction result $\hat{f}_{\theta_t}(x_i)$ of each sample on each iteration epoch $t$. Then, to reduce the influence of inaccurate prediction at the beginning of training, we estimate uncertainty with prediction after training $T_s-1$ epochs
\begin{equation}
\label{eq:approximation}
u_i \approx \frac{1}{T}\sum_{t=T_s}^{T_s+T}\kappa(y_i, \hat{f}_{\theta_t}(x_i)).
\end{equation}
We empirically show that the proposed approximation could obtain more reliable uncertainty for data with subpopulation shift in Sec.~\ref{sec:toyexp} of the Appendix.}




\textcolor{black}{To obtain reasonable importance weights, we assume that the samples with high uncertainty should be given a higher weight and vice versa. Therefore a reasonable importance weight could be linearly and positively related to the corresponding uncertainty with the following rule 
\begin{equation}
\label{eq:weight}
w_i = \eta u_i+c,
\end{equation}
where $\eta \in \mathbb{R}_{+}$ is a hyperparameter and $c \in \mathbb{R}_{+}$ is a constant that keeps the weight to be positive. In practice, we set $c$ to 1 and promising performance is achieved. The whole process for obtaining training importance weights is summarized in Algorithm \ref{algo:uncertainty}.} 

\begin{algorithm}[!htbp]
\SetAlgoNoEnd 
    \caption{The process for obtaining training importance weights.\label{algo:uncertainty}}
    \KwIn{
        Training dataset $\mathcal{D}$, sampling start epoch $T_s$, the number of sampling $T$, and upweight hyperparameter $\eta$ \;
    }
    \KwOut{The training importance weights $\mathbf{w}=[w_1, \cdots, w_n]$\;}
    \For{each iteration}{
    Train $f_\theta$ by minimizing the expected risk $\mathbb{E}\{\ell(\theta, x_i, y_i)\}$;\\
    Save the prediction results $\{\hat{f}_{\theta_t}(x_i)\}_{i=1}^{N}$ of the current epoch $t$;\\
    }
    Obtain the uncertainty of each sample with $u_i \approx \frac{1}{T}\sum_{t=T_s}^{T_s+T}\kappa(y_i, \hat{f}_{\theta_t}(x_i))$;\\
    Obtain the importance weight of each sample with $w_i = \eta u_i+c$.
\end{algorithm}

\textcolor{black}{
\textbf{Remark. (1) Epistemic and aleatoric uncertainty.} The overall uncertainty could be divided into epistemic and aleatoric uncertainty \cite{kendall2017uncertainties}. In our method, samples are weighted only based on epistemic uncertainty by sampling from the model on the training trajectory, which can be regarded as sampling from the model posterior in a more efficient way. Note that, we do not consider inherent noise (aleatoric uncertainty) since it is usually intractable to distinguish between noisy samples and minority samples from data under subpopulation shifts. \textbf{(2) Why this estimation approach could work?} Recent work has empirically shown that compared with the hard-to-classify samples, the easy-to-classify samples are learned earlier during training \cite{geifman2018biasreduced}. Meanwhile, the hard-to-classify samples are also more likely to be forgotten by the neural networks \cite{toneva2018an}. 
\textcolor{black}{The frequency of which samples are correctly classified during training can be used in confidence calibration \cite{moon2020confidence} and ensemble \cite{huang2017snapshotensembles}.}
The proposed method is also inspired by these observations and algorithms. During training, samples from the minority subpopulations are classified correctly less frequently, which corresponds to higher uncertainty. On the other hand, samples from the majority subpopulations will have lower uncertainty due to being classified correctly more frequently. In Sec.~\ref{sec:trainacc} of the Appendix, we show the accuracy of different subpopulations during training to empirically validate our claim. Meanwhile, we explain in detail why the uncertainty estimation based on historical information is chosen in Sec.~\ref{sec:justification} of the Appendix.}

\section{Theory\label{sec:theory}}
In this section, we provide a theoretical understanding of the generalization ability for \confmix in a unified framework that the \confmix can achieve a better generalization error bound than traditional IW methods without using vanilla mixup or CutMix-based methods. For simplicity, our analysis focuses on generalized linear model (GLM). The roadmap of our analysis is to first formalize the vanilla mixup and CutMix in a unified framework and then approximate the mixup loss. Finally we study the generalization bound from a Rademacher complexity perspective. To introduce the theoretical framework, we first present the basic settings.

\textbf{Basic settings.}
Our analysis mainly focuses on GLM model classes whose loss function 
$\ell$ follows $\ell(\theta, x,y) = A(\theta^{\top}x) - y\theta^{\top}x$, where $x\in \R^d$ is the input , $\theta\in \R^d$ is the parameter, $y\in \R$ is the label and $A(\cdot)$ is the log-partition function.
 
Recall the setting of subpopulation shift, we assume that the population distribution $P$ consists of $G$ different subpopulations with the $g$-th subpopulation's proportion being $k_g$ and the $g$-th subpopulation follows the distribution $P_g$. 
Specifically, we have $P=\sum_{g=1}^G k_g P_g$.
Then we denote the covariance matrix for the $g$-th subpopulation as $\Sigma_X^{g} =\E_{(x,y)\sim P_g}[xx^\top]$. 
For simplicity, we consider the case where a shared weight $w_g$ is assigned to all samples from the $g$-th subpopulation.
The main goal of our theoretical analysis is to characterize the generalization ability of the model learned using Algorithm~\ref{algo}. Formally, we focus on analyzing the upper bound of the weighted generalization error defined as:
\begin{align*}
    \operatorname{GError}(\theta) =  \E_{(x,y)\sim P}&[w(x, y)\ell(\theta, x,y)] \\
    &- \frac{1}{N}\sum_{i=1}^N w(x_i, y_i)\ell(\theta, x_i, y_i),
\end{align*}
where the function $w(x,y)$ is the weighted function to return the weight of the subpopulation to which the sample $(x,y)$ belongs.

First of all, we reformalize the vanilla mixup and CutMix in a unified framework briefly and more formal definitions are given in the appendix,
\begin{equation}
\begin{aligned}
\label{eq:IWmixup1}
    \tilde{x}_{i,j} &= M(\lambda)\odot x_i + (1-M(\lambda))\odot x_j,\\
    \tilde{y}_{i,j} &= \lambda y_i + (1-\lambda)y_j,
\end{aligned}
\end{equation}
where $\lambda$ is the ratio parameter between samples drawn from $\mathcal{D}_{\lambda}$. $\odot$ means a component-wise multiplication in vector or matrix. $M(\lambda)$ is a random variable conditioned on $\lambda$ that indicates how we mix the input, e.g., by linear interpolation~\cite{zhang2018mixup} or by a pixel mask~\cite{yun2019cutmix}. For simplicity, we denote $M(\lambda)$ by $M$ below. Under this template, two most popular mixup methods, vanilla mixup~\cite{zhang2018mixup} and CutMix~\cite{yun2019cutmix}, for $i$-th and $j$-th samples with $\lambda$ drawn from $\mathcal{D}_{\lambda}$, can be rewritten as follow
\begin{equation}
\begin{aligned}
&\tilde{x}_{i, j}^{\text{(mixup)}}=\lambda x_i+(1-\lambda) x_j,
\end{aligned}
\end{equation}
\begin{equation}
\begin{aligned}
& \tilde{x}_{i, j}^{\text{(CutMix)}}=M\odot x_i+\left(1-M\right) \odot x_j.
\end{aligned}
\end{equation}
i.e., the vanilla mixup is a special case of Equation~\ref{eq:IWmixup1} by putting $M = \lambda \overrightarrow{1}$ and the CutMix is also a special case where $M$ is a binary mask that indicates the location of the cropped box region with a relative area.

Now we present our main result in this section. 
The main theorem of our analysis provides a subpopulation-heterogeneity dependent bound for the above generalization error. This theorem is formally presented as:
\begin{theorem}
\label{thm:generalization}
Suppose $A(\cdot)$ is $L_A$-Lipschitz continuous, then there exist constants $L,B>0$ such that for any $\theta$ satisfying $\theta\in \mathcal{W}_\gamma$, the following holds with a probability of at least $1-\delta$,
\begin{align*}
    & \operatorname{GError}(\theta)\\ 
    & \le2 L \cdot L_{A}
    (\frac{1}{\sqrt{n}}(\gamma / \rho)^{1 / 4}(\sqrt{\operatorname{tr}((\Sigma_X^{(M)})^{\dagger} \Sigma_X)}+\operatorname{rank}\left(\Sigma_X\right)))\\
    &\quad +B \sqrt{\frac{\log (2 / \delta)}{2 n}},
\end{align*}
where $\gamma$ is a constant, $\Sigma_X = \sum_{g=1}^G k_gw_g \Sigma_X^{g}$, $\Sigma_X^{(M)} = \mathbb{E}[(1-M) \Sigma_X (1-M)^{\top}]$, and $\rho$ is some constant related to the data distribution, which will be formally introduced in Assumption~\ref{as:rho-retentive}. Moreover, $\mathcal{W}_\gamma = \{\theta: \mathbb{E}_x A^{''}(\theta^{\top}x) \cdot \theta^{\top}(\Sigma_X^{(M)} + x\operatorname{Var}(M)x^{\top})\theta \le \gamma\}$.

\end{theorem}

We will show later that the output of our Algorithm~\ref{algo} would yield the solution $\theta$ falls into the set $\mathcal{W}_\gamma$ and thus Theorem~\ref{thm:generalization} can provide a theoretical understanding of our algorithm.
In contrast to weighted ERM, the bound improvement of \confmix is on the term, $\sqrt{\operatorname{tr}((\Sigma_X^{(M)})^{\dagger} \Sigma_X)}+\operatorname{rank}\left(\Sigma_X\right)$, which can partially reflect the heterogeneity of the training subpopulations. 
Specifically, the term would become $\sqrt{d}$ in the weighted ERM setting (see more detailed theoretical comparisons in Appendix). 
Thus our bound can be tighter when the intrinsic dimension of data is small (i.e., $\sqrt{\operatorname{tr}((\Sigma_X^{(M)})^{\dagger} \Sigma_X)}+\operatorname{rank}\left(\Sigma_X\right)\ll d$).
Such an improvement can be achieved by tuning proper weights $w_g$ for different subpopulations.

The proof of Theorem~\ref{thm:generalization} follows this roadmap:
(1) We first show that the model learned with \confmix can fall into a
specific hypothesis set $\mathcal{W}_{\gamma}$. (2) We analyze the Rademacher complexity of the hypothesis set and obtain its complexity upper bound (Lemma~\ref{lm:rademacher}). (3) Finally, we can characterize the generalization bound by using complexity-based learning theory \cite{bartlett2002rademacher} (Theorem 8). More details of the proof can be found in Appendix. 

As we discuss in the Appendix, the weighted mixup can be seen as an approximation of a regularization term $\frac{C}{2n} [\sum_{i=1}^n w_iA^{\prime\prime}(x_i^{\top}\theta)]\theta^{\top}(\mathbb{E}(1-M)\widehat{\Sigma}_X(1-M)^{\top} + x_i\operatorname{Var}(M)x_i^{\top})\theta$ for some constant $C$, compared with the non-mixup algorithm, which motivates us to study the following hypothesis space $\mathcal{W}_\gamma$.

To further derive the generalization bound, we also need the following assumption, which is satisfied by general GLMs when $\theta$ has bounded $\ell_2$ norm and it is adopted in, e.g., \cite{arora2021dropout, zhang2021how}.
\begin{assumption}[$\rho$-retentive]
\label{as:rho-retentive}
We say the distribution of $x$ is $\rho$-retentive for some $\rho\in(0, 1/2]$ if for any non-zero vector $v\in \mathbb{R}^d$ and given the event that $\theta\in \mathcal{W}_{\gamma}$ where the $\theta$ is output by our Algorithm~\ref{alg:UMIX}, we have 
\begin{align*}
    \E_{x}^2[A^{\prime \prime}(x^{\top}v)]\ge \rho \cdot \E_{x}(v^{\top}x)^2.
\end{align*}
\end{assumption}


Finally, we can derive the Rademacher complexity of the $\mathcal{W}_{\gamma}$ and the proof of Theorem~\ref{thm:generalization} is obtained by combining  Lemma~\ref{lm:rademacher} and the Theorem 8 of \cite{bartlett2002rademacher}.
\begin{lemma}
\label{lm:rademacher}
Assume that the distribution of $x_i$ is $\rho$-retentive, i.e., satisfies the assumption~\ref{as:rho-retentive}. 
Then the Rademacher complexity of $\mathcal{W}_r$ satisfies 
\begin{align*}
    &Rad(\mathcal{W}_r)\le \\
    &\frac{1}{\sqrt{n}}(\gamma / \rho)^{1 / 4}\left(\sqrt{\operatorname{tr}\left(\left(\Sigma_X^{(M)}\right)^{\dagger} \Sigma_X\right)}+\operatorname{rank}\left(\Sigma_X\right)\right).
\end{align*}
\end{lemma}

\begin{table*}[!htbp]
  \centering
  \caption{We conduct experiments on diverse tasks and here we list the summary of the datasets used in the experiments. For each dataset, we present the number of classes, the predefined subpopulation's number, population type, input data type, the number of dataset instances, and the lowest minority subpopulation's proportion in turn. As can be seen from the table, the minimal minority subpopulation proportion in the training set is only about 2\% at most, which could significantly reduce the performance of minority subpopulations.}
    \begin{tabular}{ccccccc}
    \toprule
    Datasets & \#Classes & \#Subpopulations & Population type & Data type  & \#Size & Minimal $k_g$ \\
    \midrule
    Waterbirds & 2     & 4     & Background×Label & Synthetic Image & 11,788 & 1.17\% \\
    CelebA & 2     & 4     & Gender×Label & Face Image & 202,599 & 0.85\% \\
    CivilComments & 2     & 16     & Demographic×Label & Online Comment & 448,000 & 0.37\% \\
    FMoW & 62 & 5 & Geographical Regions & Satellite Image & 141,696 & 2.06\%\\
    Camelyon17 & 2     & 30     & Tissue Slides & Pathological Image & 455,954 & 0.47\%\\
    \bottomrule
    \end{tabular}%
  \label{tab:dataset}%
\end{table*}%

\section{Experiments\label{sec:experiments}}
In this section, we conduct multiple experiments to comprehensively evaluate the proposed algorithm. We first introduce the experimental details. Then we present the experimental results and discussions in the group-aware and group-oblivious settings respectively. Finally, we conduct an ablation study to show the key factor of performance improvement in our method.
\subsection{Experimental Setup}
The experimental setup is presented in detail in this section. Specifically, we present the used datasets details, evaluation metrics, model selection strategy, comparison methods, and other implementation details.

\textbf{Datasets}. We perform experiments on multiple benchmark datasets with subpopulation shifts, including Waterbirds \cite{Sagawa2020Distributionally}, CelebA \cite{liu2015deep}, CivilComments \cite{borkan2019nuanced}, Camelyon17\cite{bandi2018detection, koh2021wilds}, and FMoW \cite{christie2018functional,koh2021wilds}. We summarize the datasets in Table~\ref{tab:dataset} and the details of these datasets are as follows:
\begin{itemize}
    \item \textbf{WaterBirds}. The task of this dataset is to distinguish whether the bird is a waterbird or a landbird. According to the background and label of an image, this dataset has four predefined subpopulations, i.e., ``landbirds on land'', ``landbirds on water'', ``waterbirds on land`` , and ``waterbirds on water''. In the training set, the largest subpopulation is ``landbirds on land'' with 3,498 samples, while the lowest subpopulation is ``landbirds on water'' with only 56 samples.
    \item \textbf{CelebA}. CelebA is a well-known large-scale face dataset. Same as previous works \cite{Sagawa2020Distributionally,liu2021just}, we employ this dataset to predict the color of the human hair as ``blond'' or ``not blond''. There are four predefined subpopulations based on gender and hair color, i.e., ``dark hair, female'', ``dark hair, male'', ``blond hair, female'' and ``blond hair, male'' with 71,629, 66,874, 22,880, and 1,387 training samples respectively. 
    \item \textbf{CivilComments}. For this dataset, the task is to classify whether an online comment is toxic or not. According to the demographic identities (e.g., Female, Male, and White) and labels, 16 overlapping subpopulations can be defined. We use 269,038, 45,180, and 133,782 samples as training, validation, and test datasets respectively.  In the group-aware setting, if one sample belongs to  multiple subpopulations, we use the least subpopulation (with the smallest training number) as the group label for that example to obtain the importance weights.
    \item \textbf{FMoW}. FMoW is a satellite image dataset and the task is to distinguish the building or land use categories of each image. The training data was taken in five different geographical regions (Africa, the Americas, Oceania, Asia, and Europe). Since the distribution of satellite images in different geographic regions is different due to different development statuses and the training number of different geographical regions is different, we directly use the geographical region information as the subpopulation label.
    \item \textbf{Camelyon17}. Camelyon17 is a pathological image dataset with over 450, 000 lymph-node scans used to distinguish whether there is cancer tissue in a patch. The training data is drawn from 30 tissue slides from three different hospitals, while the validation and test data are sampled from another 10 tissue slides from other hospitals respectively. Due to the different coloring methods, even the same hospital samples have different distributions. Therefore, we use the tissue slides information as the subpopulation label in the group-aware setting.
\end{itemize}


\textbf{Evaluation metrics}.
To be same as previous works \cite{yao2022improving,piratla2022focus} and 
\href{https://wilds.stanford.edu/leaderboard/}{WILDs leaderboard}\cite{koh2021wilds}, we report the average accuracy over 10 different random seeds on the Camelyon17 dataset. On other datasets, we repeat experiments over 3 times and report the average and worst-case accuracy among all subpopulations.  Besides, to better show the performance gap between different subpopulations, we also present the gap between the average accuracy and worst-case accuracy. Note that, the trade-off between the average and worst-case accuracy is a well-known challenge \cite{hashimoto2018fairness,kumar2022calibrated} and in this paper we emphasize worst-case accuracy, which is more important than the average accuracy in some application scenarios. For example, in fairness-related applications, we should pay more attention to the performance of the minority groups to reduce the gap between the majority groups and ensure the fairness of the machine learning decision system.
\begin{table*}[!htbp]
  \centering
  \caption{Comparison results in the group aware setting with current state-of-the-art algorithms. For better presentation, the best results are in \textcolor{blue}{\textbf{bold}} and \textcolor{blue}{\textbf{blue}}. Meanwhile, the second and third results are in \textcolor{brown}{\textbf{bold}} and \textcolor{brown}{\textbf{brown}}. \emph{Note that the proposed method focuses on improving the model's worst-case accuracy to improve model subpopulation shift robustness rather than the average accuracy.} The proposed method consistently outperforms all comparison methods in worst-case accuracy on all datasets.}
    \begin{tabular}{l|cccccccccccccccccccccccccccccccccccccccccc|}
    \toprule
    Datasets/ & \multicolumn{21}{c|}{Waterbirds} & \multicolumn{21}{c|}{CelebA} \\
    ACC(\%) & \multicolumn{7}{p{5.25em}<{\centering}}{Avg.$(\uparrow)$}         & \multicolumn{7}{p{5.25em}<{\centering}}{Worst$(\uparrow)$}                     & \multicolumn{7}{p{3.5em}<{\centering}|}{Gap$(\downarrow)$}                       & \multicolumn{7}{p{5.25em}<{\centering}}{Avg.$(\uparrow)$}                   & \multicolumn{7}{p{5.25em}<{\centering}}{Worst$(\uparrow)$}                     & \multicolumn{7}{p{3.5em}<{\centering}|}{Gap$(\downarrow)$} \\
    \midrule
    ERM   & \multicolumn{7}{p{5.25em}<{\centering}}{\textcolor{blue}{\textbf{97.0$_{(0.2)}$}}}                 & \multicolumn{7}{p{5.25em}<{\centering}}{63.7$_{(1.9)}$}                 & \multicolumn{7}{p{3.5em}<{\centering}|}{33.3}                              & \multicolumn{7}{p{5.25em}<{\centering}}{\textcolor{blue}{\textbf{94.9$_{(0.2)}$}}}                 & \multicolumn{7}{p{5.25em}<{\centering}}{47.8$_{(3.7)}$}                 & \multicolumn{7}{p{3.5em}<{\centering}|}{47.1} \\
    \midrule
    IRM   & \multicolumn{7}{p{5.25em}<{\centering}}{87.5$_{(0.7)}$}                 & \multicolumn{7}{p{5.25em}<{\centering}}{75.6$_{(3.1)}$}                 & \multicolumn{7}{p{3.5em}<{\centering}|}{11.9}                              & \multicolumn{7}{p{5.25em}<{\centering}}{\textcolor{brown}{\textbf{94.0$_{(0.4)}$}}}                 & \multicolumn{7}{p{5.25em}<{\centering}}{77.8$_{(3.9)}$}                 & \multicolumn{7}{p{3.5em}<{\centering}|}{16.2} \\
    IB-IRM & \multicolumn{7}{p{5.25em}<{\centering}}{88.5$_{(0.6)}$}                 & \multicolumn{7}{p{5.25em}<{\centering}}{76.5$_{(1.2)}$}                 & \multicolumn{7}{p{3.5em}<{\centering}|}{12.0}                                & \multicolumn{7}{p{5.25em}<{\centering}}{93.6$_{(0.3)}$}                 & \multicolumn{7}{p{5.25em}<{\centering}}{85.0$_{(1.8)}$}                 & \multicolumn{7}{p{3.5em}<{\centering}|}{8.6} \\
    CORAL & \multicolumn{7}{p{5.25em}<{\centering}}{90.3$_{(1.1)}$}                 & \multicolumn{7}{p{5.25em}<{\centering}}{79.8$_{(1.8)}$}                 & \multicolumn{7}{p{3.5em}<{\centering}|}{10.5}                              & \multicolumn{7}{p{5.25em}<{\centering}}{\textcolor{brown}{\textbf{93.8$_{(0.3)}$}}}                 & \multicolumn{7}{p{5.25em}<{\centering}}{76.9$_{(3.6)}$}                 & \multicolumn{7}{p{3.5em}<{\centering}|}{16.9} \\
    Fish  & \multicolumn{7}{p{5.25em}<{\centering}}{85.6$_{(0.4)}$}                 & \multicolumn{7}{p{5.25em}<{\centering}}{64.0$_{(0.3)}$}                 & \multicolumn{7}{p{3.5em}<{\centering}|}{21.6}                              & \multicolumn{7}{p{5.25em}<{\centering}}{93.1$_{(0.3)}$}                 & \multicolumn{7}{p{5.25em}<{\centering}}{61.2$_{(2.5)}$}                 & \multicolumn{7}{p{3.5em}<{\centering}|}{31.9} \\
    \midrule
    GroupDRO  & \multicolumn{7}{p{5.25em}<{\centering}}{\textcolor{brown}{\textbf{91.8$_{(0.3)}$}}}                 & \multicolumn{7}{p{5.25em}<{\centering}}{\textcolor{brown}{\textbf{90.6$_{(1.1)}$}}}                 & \multicolumn{7}{p{3.5em}<{\centering}|}{\textcolor{blue}{\textbf{1.2}}}                               & \multicolumn{7}{p{5.25em}<{\centering}}{92.1$_{(0.4)}$}                 & \multicolumn{7}{p{5.25em}<{\centering}}{87.2$_{(1.6)}$}                 & \multicolumn{7}{p{3.5em}<{\centering}|}{4.9} \\
    V-REX & \multicolumn{7}{p{5.25em}<{\centering}}{88.0$_{(1.0)}$}                 & \multicolumn{7}{p{5.25em}<{\centering}}{73.6$_{(0.2)}$}                 & \multicolumn{7}{p{3.5em}<{\centering}|}{14.4}                              & \multicolumn{7}{p{5.25em}<{\centering}}{92.2$_{(0.1)}$}                 & \multicolumn{7}{p{5.25em}<{\centering}}{86.7$_{(1.0)}$}                 & \multicolumn{7}{p{3.5em}<{\centering}|}{5.5} \\
    CGD   & \multicolumn{7}{p{5.25em}<{\centering}}{91.3$_{(0.6)}$}                 & \multicolumn{7}{p{5.25em}<{\centering}}{88.9$_{(0.8)}$}                 & \multicolumn{7}{p{3.5em}<{\centering}|}{\textcolor{brown}{\textbf{2.4}}}                               & \multicolumn{7}{p{5.25em}<{\centering}}{92.5$_{(0.2)}$}                 & \multicolumn{7}{p{5.25em}<{\centering}}{\textcolor{brown}{\textbf{90.0$_{(0.8)}$}}}                 & \multicolumn{7}{p{3.5em}<{\centering}|}{\textcolor{brown}{\textbf{2.5}}} \\
    \midrule
    DomainMix  & \multicolumn{7}{p{5.25em}<{\centering}}{76.4$_{(0.3)}$}                 & \multicolumn{7}{p{5.25em}<{\centering}}{53.0$_{(1.3)}$}                 & \multicolumn{7}{p{3.5em}<{\centering}|}{23.4}                              & \multicolumn{7}{p{5.25em}<{\centering}}{93.4$_{(0.1)}$}                 & \multicolumn{7}{p{5.25em}<{\centering}}{65.6$_{(1.7)}$}                 & \multicolumn{7}{p{3.5em}<{\centering}|}{27.8} \\
    LISA  & \multicolumn{7}{p{5.25em}<{\centering}}{91.8$_{(0.3)}$}                 & \multicolumn{7}{p{5.25em}<{\centering}}{\textcolor{brown}{\textbf{89.2$_{(0.6)}$}}}                 & \multicolumn{7}{p{3.5em}<{\centering}|}{2.6}                               & \multicolumn{7}{p{5.25em}<{\centering}}{92.4$_{(0.4)}$}                 & \multicolumn{7}{p{5.25em}<{\centering}}{\textcolor{brown}{\textbf{89.3$_{(1.1)}$}}}                 & \multicolumn{7}{p{3.5em}<{\centering}|}{\textcolor{brown}{\textbf{3.1}}} \\
    \midrule
    Ours  & \multicolumn{7}{p{5.25em}<{\centering}}{\textcolor{brown}{\textbf{93.5$_{(0.2)}$}}}                 & \multicolumn{7}{p{5.25em}<{\centering}}{\textcolor{blue}{\textbf{91.6$_{(0.2)}$}}} & \multicolumn{7}{p{3.5em}<{\centering}|}{\textcolor{brown}{\textbf{1.9}}} & \multicolumn{7}{p{5.25em}<{\centering}}{91.3$_{(0.5)}$}                 & \multicolumn{7}{p{5.25em}<{\centering}}{\textcolor{blue}{\textbf{90.6$_{(0.6)}$}}}                 & \multicolumn{7}{p{3.5em}<{\centering}|}{\textcolor{blue}{\textbf{0.7}}} \\
    \midrule
    \midrule
    Datasets/ & \multicolumn{18}{c|}{CivilComments}& \multicolumn{18}{c|}{FMoW} & \multicolumn{6}{p{4em}<{\centering}|}{Camely.} \\
    ACC(\%) & \multicolumn{6}{p{4em}<{\centering}}{Avg.$(\uparrow)$}           & \multicolumn{6}{p{4em}<{\centering}}{Worst$(\uparrow)$}             & \multicolumn{6}{p{4em}<{\centering}|}{Gap$(\downarrow)$}               & \multicolumn{6}{p{4em}<{\centering}}{Avg.$(\uparrow)$}           & \multicolumn{6}{p{4em}<{\centering}}{Worst$(\uparrow)$}             & \multicolumn{6}{p{4em}<{\centering}|}{Gap$(\downarrow)$}               & \multicolumn{6}{p{4em}<{\centering}|}{Avg.$(\uparrow)$} \\
    \midrule
    ERM   & \multicolumn{6}{p{4em}<{\centering}}{\textcolor{blue}{\textbf{92.2$_{(0.1)}$}}}         & \multicolumn{6}{p{4em}<{\centering}}{56.0$_{(3.6)}$}         & \multicolumn{6}{c|}{36.2}                      & \multicolumn{6}{p{4em}<{\centering}}{\textcolor{blue}{\textbf{53.0$_{(0.6)}$}}}         & \multicolumn{6}{p{4em}<{\centering}}{32.3$_{(1.3)}$}         & \multicolumn{6}{c|}{20.7}                      & \multicolumn{6}{p{4em}<{\centering}|}{70.3$_{(6.4)}$} \\
    \midrule
    IRM   & \multicolumn{6}{p{4em}<{\centering}}{88.8$_{(0.7)}$}         & \multicolumn{6}{p{4em}<{\centering}}{66.3$_{(2.1)}$}         & \multicolumn{6}{c|}{22.5}                      & \multicolumn{6}{p{4em}<{\centering}}{50.8$_{(0.1)}$}         & \multicolumn{6}{p{4em}<{\centering}}{30.0$_{(1.4)}$}         & \multicolumn{6}{c|}{20.8}                      & \multicolumn{6}{p{4em}<{\centering}|}{64.2$_{(8.1)}$} \\
    IB-IRM & \multicolumn{6}{p{4em}<{\centering}}{89.1$_{(0.3)}$}         & \multicolumn{6}{p{4em}<{\centering}}{65.3$_{(1.5)}$}         & \multicolumn{6}{c|}{23.8}                      & \multicolumn{6}{p{4em}<{\centering}}{49.5$_{(0.5)}$}         & \multicolumn{6}{p{4em}<{\centering}}{28.4$_{(0.9)}$}         & \multicolumn{6}{c|}{21.1}                      & \multicolumn{6}{p{4em}<{\centering}|}{68.9$_{(6.1)}$} \\
    CORAL & \multicolumn{6}{p{4em}<{\centering}}{88.7$_{(0.5)}$}         & \multicolumn{6}{p{4em}<{\centering}}{65.6$_{(1.3)}$}         & \multicolumn{6}{c|}{23.1}                      & \multicolumn{6}{p{4em}<{\centering}}{50.5$_{(0.4)}$}         & \multicolumn{6}{p{4em}<{\centering}}{31.7$_{(1.2)}$}         & \multicolumn{6}{c|}{18.8}                      & \multicolumn{6}{p{4em}<{\centering}|}{59.5$_{(7.7)}$} \\
    Fish  & \multicolumn{6}{p{4em}<{\centering}}{89.8$_{(0.4)}$}         & \multicolumn{6}{p{4em}<{\centering}}{\textcolor{brown}{\textbf{71.1$_{(0.4)}$}}}         & \multicolumn{6}{c|}{\textcolor{brown}{\textbf{18.7}}}                      & \multicolumn{6}{p{4em}<{\centering}}{51.8$_{(0.3)}$}         & \multicolumn{6}{p{4em}<{\centering}}{\textcolor{brown}{\textbf{34.6$_{(0.2)}$}}}         & \multicolumn{6}{c|}{\textcolor{brown}{\textbf{17.2}}}                      & \multicolumn{6}{p{4em}<{\centering}|}{\textcolor{brown}{\textbf{74.7$_{(7.1)}$}}} \\
    \midrule
    GroupDRO  & \multicolumn{6}{p{4em}<{\centering}}{89.9$_{(0.5)}$}         & \multicolumn{6}{p{4em}<{\centering}}{70.0$_{(2.0)}$}         & \multicolumn{6}{c|}{19.9}                      & \multicolumn{6}{p{4em}<{\centering}}{\textcolor{brown}{\textbf{52.1$_{(0.5)}$}}}         & \multicolumn{6}{p{4em}<{\centering}}{30.8$_{(0.8)}$}         & \multicolumn{6}{c|}{21.3}                      & \multicolumn{6}{p{4em}<{\centering}|}{68.4$_{(7.3)}$} \\
    V-REX & \multicolumn{6}{p{4em}<{\centering}}{90.2$_{(0.3)}$}         & \multicolumn{6}{p{4em}<{\centering}}{64.9$_{(1.2)}$}         & \multicolumn{6}{c|}{25.3}                      & \multicolumn{6}{p{4em}<{\centering}}{48.0$_{(0.6)}$}         & \multicolumn{6}{p{4em}<{\centering}}{27.2$_{(0.8)}$}         & \multicolumn{6}{c|}{20.8}                      & \multicolumn{6}{p{4em}<{\centering}|}{71.5$_{(8.3)}$} \\
    CGD   & \multicolumn{6}{p{4em}<{\centering}}{89.6$_{(0.4)}$}         & \multicolumn{6}{p{4em}<{\centering}}{69.1$_{(1.9)}$}         & \multicolumn{6}{c|}{20.5}                      & \multicolumn{6}{p{4em}<{\centering}}{50.6$_{(1.4)}$}         & \multicolumn{6}{p{4em}<{\centering}}{32.0$_{(2.3)}$}         & \multicolumn{6}{c|}{18.6}                      & \multicolumn{6}{p{4em}<{\centering}|}{69.4$_{(7.8)}$} \\
    \midrule
    DomainMix  & \multicolumn{6}{p{4em}<{\centering}}{\textcolor{brown}{\textbf{90.9$_{(0.4)}$}}}         & \multicolumn{6}{p{4em}<{\centering}}{63.6$_{(2.5)}$}         & \multicolumn{6}{c|}{27.3}                      & \multicolumn{6}{p{4em}<{\centering}}{51.6$_{(0.2)}$}         & \multicolumn{6}{p{4em}<{\centering}}{34.2$_{(0.8)}$}         & \multicolumn{6}{c|}{17.4}                      & \multicolumn{6}{p{4em}<{\centering}|}{69.7$_{(5.5)}$} \\
    LISA  & \multicolumn{6}{p{4em}<{\centering}}{89.2$_{(0.9)}$}         & \multicolumn{6}{p{4em}<{\centering}}{\textcolor{brown}{\textbf{72.6$_{(0.1)}$}}}         & \multicolumn{6}{c|}{\textcolor{blue}{\textbf{16.6}}}                      & \multicolumn{6}{p{4em}<{\centering}}{\textcolor{brown}{\textbf{52.8$_{(0.9)}$}}}         & \multicolumn{6}{p{4em}<{\centering}}{\textcolor{brown}{\textbf{35.5$_{(0.7)}$}}}         & \multicolumn{6}{c|}{\textcolor{brown}{\textbf{17.3}}}                      & \multicolumn{6}{p{4em}<{\centering}|}{\textcolor{blue}{\textbf{77.1$_{(6.5)}$}}} \\
    \midrule
    Ours  & \multicolumn{6}{p{4em}<{\centering}}{\textcolor{brown}{\textbf{90.6$_{(0.4)}$}}}         & \multicolumn{6}{p{4em}<{\centering}}{\textcolor{blue}{\textbf{73.3$_{(0.3)}$}}}         & \multicolumn{6}{c|}{\textcolor{brown}{\textbf{17.3}}}                      & \multicolumn{6}{p{4em}<{\centering}}{49.5$_{(1.5)}$}         & \multicolumn{6}{p{4em}<{\centering}}{\textcolor{blue}{\textbf{38.2$_{(1.7)}$}}}         & \multicolumn{6}{c|}{\textcolor{blue}{\textbf{11.3}}}                      & \multicolumn{6}{p{4em}<{\centering}|}{\textcolor{brown}{\textbf{77.0$_{(5.5)}$}}} \\
    \bottomrule
    \end{tabular}%
  \label{tab:exp-group-aware}%
\end{table*}%

\textbf{Model selection}. Following prior works \cite{liu2021just,zhai2021doro}, in the both group-aware and group-oblivious settings, we assume the group information of validation samples are available and select the best model based on worst-case accuracy among all subpopulations on the validation set. In the group-oblivious setting, we also conduct model selection based on the average accuracy to show the impact of validation group label information in our method.

\textbf{Comparison methods in the group-aware setting}. In order to demonstrate the performance of the proposed method, we compare our method with multiple methods that use training group labels during training. The details are as follows:
\begin{itemize}
    \item Invariant risk minimization and other closely related methods, including, IRM \cite{arjovsky2019invariant} tries to learn the invariant correlations between different training distributions; IB-IRM \cite{ahuja2021invariance} improves IRM with information bottleneck; CORAL \cite{sun2016deep} aims to align the correlations of deep neural networks layer activations by a nonlinear transformation, and Fish \cite{shi2021gradient} learns deep neural network with invariant gradient direction for different populations which could promote invariant predictions.
    \item Group distributionally robust optimization and the follow-up algorithms, including, GroupDRO \cite{Sagawa2020Distributionally} is an online optimization algorithm that dynamically assigns importance weights during training to improve the worst-case performance across all the subpopulations; V-REx \cite{krueger2021out} is an extension of distributionally robust optimization by conducting robust optimization over a perturbation set of extrapolated domain; CGD \cite{piratla2022focus} assigns importance weights to each subpopulation so that the model could be trained on the direction leading to the largest decrease in average training loss.
    \item Mixup-based methods. We also compare our method with DomainMix \cite{xu2020adversarial} and LISA \cite{yao2022improving}, where DomainMix introduces adversarial training to mix the samples to synthesize novel examples while LISA mixes the samples within the same subpopulation or same label.
\end{itemize}
Further, we also compare the proposed method with empirical risk minimization (ERM).

\textbf{Comparison methods in the group-oblivious setting}. In the group-oblivious setting, we compare the proposed method with multiple methods that do not require training group information, including
\begin{itemize}
    \item ERM trains the model using standard empirical risk minimization;
    \item Focal loss \cite{lin2017focal} downweights the well-classified examples' loss according to the current classification confidences.
    \item CVaRDro and $\chi^2$Dro \cite{levy2020large} minimize the loss over the worst-case distribution in a neighborhood of the empirical training distribution.
    \item CVaRDoro and $\chi^2$Doro \cite{zhai2021doro} are outlier robust version of CVarDro and $\chi^2$Dro respectively, which dynamically remove the outlier examples during training.
    \item JTT \cite{liu2021just} is a two-stage reweighting algorithm to improve the subpopulation shift robustness, which first constructs an error set and then upweights the samples in the error set.
\end{itemize}

\begin{table*}[!tbp]
  \centering
  \caption{Comparison results with other methods in the group-oblivious setting. We report the mean and corresponding standard deviations in parentheses for multiple random seeds. The best results are in \textcolor{blue}{\textbf{bold}} and \textcolor{blue}{\textbf{blue}}. Asterisks($*$) indicate that standard deviations were not provided in the original paper. \emph{Note that the proposed method focuses on improving the model's worst-case accuracy to improve model subpopulation shift robustness rather than the average accuracy.} It can be seen that the proposed method achieves the best performance on all datasets.}
    \begin{tabular}{l|ccc|ccc|ccc|c|}
    \toprule
    Datasets/ & \multicolumn{3}{c|}{Waterbirds} & \multicolumn{3}{c|}{CelebA} & \multicolumn{3}{c|}{CivilComments} & Camely. \\
    ACC(\%) & Avg.$_{(\uparrow)}$ & Worst$_{(\uparrow)}$ & \multicolumn{1}{c|}{Gap$_{(\downarrow)}$} & Avg.$_{(\uparrow)}$ & Worst$_{(\uparrow)}$ & \multicolumn{1}{c|}{Gap$_{(\downarrow)}$} & Avg.$_{(\uparrow)}$ & Worst$_{(\uparrow)}$ &
    \multicolumn{1}{c|}{Gap$_{(\downarrow)}$} & Avg.$_{(\uparrow)}$ \\
    \midrule
    ERM   & \textcolor{blue}{\textbf{97.0$_{(0.2)}$}} & 63.7$_{(1.9)}$ & 33.3  & \textcolor{blue}{\textbf{94.9$_{(0.2)}$}} & 47.8$_{(3.7)}$ & 47.1  & \textcolor{blue}{\textbf{92.2$_{(0.1)}$}} & 56.0$_{(3.6)}$ & 36.2  & 70.3$_{(6.4)}$ \\
    \midrule
    Focal Loss & 87.0$_{(0.5)}$ & 73.1$_{(1.0)}$ & 13.9  & 88.4$_{(0.3)}$ & 72.1$_{(3.8)}$ & 16.3  & 91.2$_{(0.5)}$ & 60.1$_{(0.7)}$ & 31.1  & 68.1$_{(4.8)}$ \\
    CVaRDro & 90.3$_{(1.2)}$ & 77.2$_{(2.2)}$ & 13.1  & 86.8$_{(0.7)}$ & 76.9$_{(3.1)}$ & 9.9   & 89.1$_{(0.4)}$ & 62.3$_{(0.7)}$ & 26.8  & 70.5$_{(5.1)}$ \\
    CVaRDoro & 91.5$_{(0.7)}$ & 77.0$_{(2.8)}$ & 14.5  & 89.6$_{(0.4)}$ & 75.6$_{(4.2)}$ & 14.0    & 90.0$_{(0.4)}$ & 64.1$_{(1.4)}$ & 25.9  & 67.3$_{(7.2)}$ \\
    $\chi^{2}$Dro & 88.8$_{(1.5)}$ & 74.0$_{(1.8)}$ & 14.8  & 87.7$_{(0.3)}$ & 78.4$_{(3.4)}$ & 9.3   & 89.4$_{(0.7)}$ & 64.2$_{(1.3)}$ & 25.2  & 68.0$_{(6.7)}$ \\
    $\chi^{2}$Doro & 89.5$_{(2.0)}$ & 76.0$_{(3.1)}$ & 13.5  & 87.0$_{(0.6)}$ & 75.6$_{(3.4)}$ & 11.4  & 90.1$_{(0.5)}$ & 63.8$_{(0.8)}$ & 26.3  & 68.0$_{(7.5)}$ \\
    JTT   & 93.6$_{(*)}$ & 86.0$_{(*)}$ & 7.6   & 88.0$_{(*)}$ & 81.1$_{(*)}$ & 6.9   & 90.7$_{(0.3)}$ & 67.4$_{(0.5)}$ & 23.3  & 69.1$_{(6.4)}$ \\
    \midrule
    Ours  & 93.0$_{(0.5)}$ & \textcolor{blue}{\textbf{90.0$_{(1.1)}$}} & \textcolor{blue}{\textbf{3.0}}     & 90.1$_{(0.4)}$ & \textcolor{blue}{\textbf{85.3$_{(4.1)}$}} & \textcolor{blue}{\textbf{4.8}}   & 90.6$_{(0.4)}$ & \textcolor{blue}{\textbf{70.1$_{(0.9)}$}} & \textcolor{blue}{\textbf{20.5}}  & \textcolor{blue}{\textbf{75.1$_{(5.9)}$}} \\
    \bottomrule
    \end{tabular}%
  \label{tab:exp-group-oblivious}%
\end{table*}%

\begin{table*}[!htbp]
  \centering
  \caption{Experimental results when the group labels in the validation set are available or not. We report the average accuracy and worst-case accuracy of the models, while reporting the variation in performance of the model with and without group labels on the validation set. }
    \begin{tabular}{c|ccc|ccc|}
    \toprule
    \multicolumn{1}{c|}{{Group labels in}} & \multicolumn{3}{c|}{Waterbirds} & \multicolumn{3}{c|}{CelebA} \\
    \multicolumn{1}{c|}{{validation set?}}  & Average ${(\uparrow)}$ & Worst-case ${(\uparrow)}$ &Gap ${(\downarrow)}$& Average $_{(\uparrow)}$ & Worst-case $_{(\uparrow)}$ & Gap ${(\downarrow)}$\\
    \midrule
    Yes     & 93.0$_{(0.5)}$ & 90.0$_{(1.1)}$ &3.0 & 90.1$_{(0.4)}$ & 85.3$_{(4.1)}$ & 4.8\\
    No     & 93.6$_{(0.5)}$ & 88.9$_{(0.8)}$ & 4.7 &90.4$_{(0.4)}$ & 84.6$_{(0.8)}$ & 5.8\\
    \midrule
    Variations & +0.6 & -1.1 & +1.7 & +0.3 & -0.7& +1.0\\
    \bottomrule
    \end{tabular}%
  \label{tab:withoutvalidation}%
\end{table*}

\textbf{Comparison methods in the ablation study}. We compare our method with vanilla mixup and in-group mixup, where vanilla mixup is performed on any pair of samples and in-group mixup (IGMix) is performed on the samples with the same labels and from the same subpopulations. Meanwhile, we also compare the proposed method with the naive importance-weighting (IW) method.

\textbf{Implementation details.} 
Within each dataset, we keep the same model architecture as in previous work \cite{yao2022improving}:
ResNet-50 pretrained on ImageNet \cite{he2016deep} for Waterbirds and CelebA, DistilBERT \cite{devlin2018bert} for CivilComments, DenseNet-121 pretrained on ImageNet\cite{huang2019convolutional} for FMoW, and DenseNet-121 without pretraining for Camelyon17. For DistilBERT, we employ the HuggingFace \cite{wolf2019huggingface} implementation and start from the pre-trained weights. For the proposed method, we implement it on the   \href{https://github.com/p-lambda/wilds}{codestack} released with the WILDS datasets \cite{koh2021wilds}. For most comparison methods, we directly use the results in previous work \cite{yao2022improving}, original papers, or reported in the \texttt{WILDs leaderboard}. For the rest of the comparison methods, we reimplement them according to the original papers. We employ vanilla mixup on WaterBirds and Camelyon17 datasets, CutMix \cite{yun2019cutmix} on CelebA and FMoW datasets, and manifoldmix \cite{verma2019manifold} on CivilComments dataset. For all approaches, we tune all hyperparameters with AutoML toolkit NNI \cite{nni2021} based on validation performance.

\subsection{Group-aware experimental results}
We conduct extensive comparison experiments in the group-ware setting on five datasets to verify the robustness of the proposed model against subpopulation shifts. The experimental results are reported in Tab.~\ref{tab:exp-group-aware}. From the experimental results, we have the following observations. 
\begin{itemize}
    \item (1) The proposed method consistently outperforms all comparison methods in terms of worst-case accuracy on all the datasets. For example, on the Waterbirds dataset, the proposed method achieves the worst-case accuracy of 91.6\%, which outperforms all the other comparison methods. Meanwhile, on the FMoW dataset, the proposed method outperforms the second method by 2.7\% in terms of worst-case accuracy. 
    \item  (2) ERM can achieve the best average accuracy on almost all used datasets, but its corresponding worst-case accuracy performance is much lower than the proposed method and comparison methods. The reason for this phenomenon is that the ERM algorithm may over-focus on the majority subpopulation's samples and learn the spurious correlations, therefore improving the average performance of the model based on the spurious correlation. However, due to over-reliance on spurious correlations, the performance of ERM will degrade significantly in the minority subpopulations. While our method can significantly improve the performance of the minority subpopulations by introducing importance-weighted mixup. For example, on the Waterbirds dataset and CelebA dataset, the worst-case accuracy of the proposed method is 27.9\% and 42.8\% higher than ERM, respectively.
    \item (3) Compared with other comparison methods, the proposed method is also competitive in terms of average accuracy. For example, on the Waterbirds and CivilComments dataset, the average accuracy of the proposed method is in the top three.
    \item (4) The proposed method could significantly reduce the gap between the average and worst-case accuracy, which is important for the fairness of the model. Specifically, on the CelebA and FMoW datasets, the proposed method has the lowest average and worst-case accuracy gap. On the Waterbirds and CivilComments, the proposed method has the second-lowest average and worst-case accuracy gap. Moreover, ERM usually has the largest average and worst-case accuracy gap, which significantly affects the application of ERM to fairness-sensitive tasks.
    \item (5) By performing an importance-reweighted mixup and exploring the minority subpopulations, the proposed method can also enhance the robustness of the model to out-of-distribution data. Specifically, the proposed method achieves the second-best performance i.e., 77.0\% average accuracy, on the Camelyon17 dataset. 
\end{itemize}

\begin{figure*}[!htbp]
\centering
\subfigure[Waterbirds]{
\begin{minipage}[t]{0.45\linewidth}
\centering
\includegraphics[width=1\linewidth,height=0.7\linewidth]{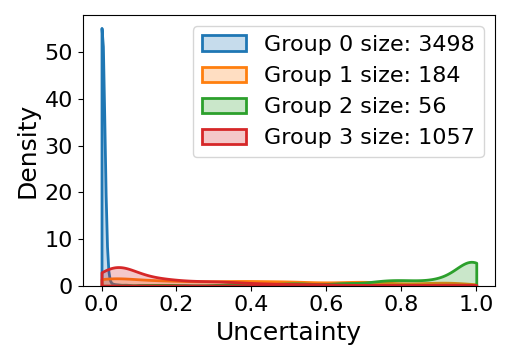}
\centering
\end{minipage}}
\subfigure[CelebA]{
\begin{minipage}[t]{0.45\linewidth}
\centering
\includegraphics[width=1\linewidth,height=0.7\linewidth]{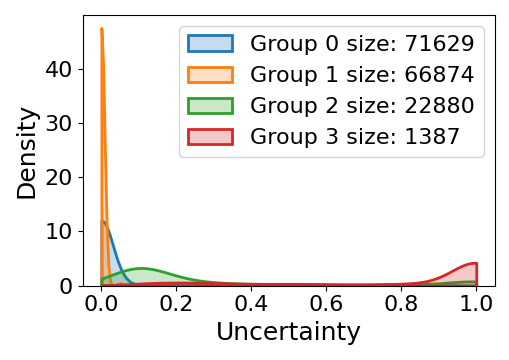}
\end{minipage}}
\caption{\label{fig:quali}Visualization of the obtained uncertainty with kernel density estimation on Waterbirds and CelebA datasets, where group size refers to the sample number of the group in the training dataset. It can be seen from the experimental results that the proposed uncertainty estimation method can better characterize the subpopulation distribution of the training set and find hard samples belonging to the minority subpopulations. }
\end{figure*}
\begin{table*}[!htbp]
  \centering
  \caption{Comparison results with ERM, importance weighting, and mixup-based methods. From the experimental results, it can be seen that the proposed method significantly improves the worst-case accuracy of the model compared with mixup and importance weighting algorithms.}
    \begin{tabular}{lc|ccc|ccc|ccc|}
    \toprule
    \footnotesize{Datasets}/ &\footnotesize{Experimental} & \multicolumn{3}{c|}{Waterbirds} & \multicolumn{3}{c|}{CelebA} & \multicolumn{3}{c|}{CivilComments}  \\
    ACC(\%) & \footnotesize{Setting} & Avg.$_{(\uparrow)}$ & Worst$_{(\uparrow)}$ & \multicolumn{1}{c|}{Gap$_{(\downarrow)}$} & Avg.$_{(\uparrow)}$ & Worst$_{(\uparrow)}$ & \multicolumn{1}{c|}{Gap$_{(\downarrow)}$} & Avg.$_{(\uparrow)}$ & Worst$_{(\uparrow)}$ &
    \multicolumn{1}{c|}{Gap$_{(\downarrow)}$} \\
    \midrule
    ERM & oblivious & \textcolor{black}{{97.0$_{(0.2)}$}} & 63.7$_{(1.9)}$ & 33.3  & 94.9$_{(0.2)}$ & 47.8$_{(3.7)}$ & 47.1  & 92.2$_{(0.1)}$ & 56.0$_{(3.6)}$ & 36.2\\
    \midrule
    mixup& oblivious & 81.0$_{(0.2)}$ & 56.2$_{(0.2)}$ &24.8 & 95.8$_{(0.0)}$ & 46.4$_{(0.5)}$ & 49.4 & 90.8$_{(0.8)}$ & 67.2$_{(1.2)}$ & 23.6\\ 
    IGMix & aware &  88.7$_{(0.3)}$ & 68.0$_{(0.4)}$ &20.7 & 95.2$_{(0.3)}$ & 58.3$_{(0.9)}$ & 36.9 & 90.8$_{(0.6)}$ & 69.2$_{(0.8)}$ & 21.6\\
    IW & aware &  95.1$_{(0.3)}$ & 88.0$_{(1.3)}$ &7.1 & 92.9$_{(0.2)}$ & 83.3$_{(2.8)}$ & 9.6 & 89.8$_{(0.5)}$ & 69.2$_{(0.9)}$ & 20.6 \\
    \midrule
    Ours & oblivious & 93.0$_{(0.5)}$ & 90.0$_{(1.1)}$ & 3.0     & 90.1$_{(0.4)}$ & 85.3$_{(4.1)}$ & 4.8   & 90.6$_{(0.4)}$ & 70.1$_{(0.9)}$ & 20.5 \\
    Ours & aware & 93.5$_{(0.5)}$ & 91.6$_{(0.2)}$ & 1.9     & 91.3$_{(0.5)}$ & 90.6$_{(0.6)}$ & 0.7   & 90.6$_{(0.4)}$ & 73.3$_{(0.3)}$ & 17.3 \\
    \bottomrule
\label{tab:ablation}%
\end{tabular}%
\end{table*}%
\subsection{Group-oblivious experimental results}
In this section, we conduct various experiments in the group-oblivious setting on multiple datasets with subpopulation shifts to answer the following questions. \textbf{Q1 Effectiveness} (\uppercase\expandafter{\romannumeral1}). In the group-oblivious setting, does the proposed method outperform other algorithms? 
\textbf{Q2 Effectiveness} 
(\uppercase\expandafter{\romannumeral2}). How does UMIX perform without the group labels in the validation set?
\textbf{Q3 Qualitative analysis}. Are the obtained uncertainties of the training samples trustworthy?

\textbf{Q1 Effectiveness} (\uppercase\expandafter{\romannumeral1}). {When the group information in the training set is not available, the proposed method could estimate the uncertainty of each training sample and then construct sample-wise importance weights. Therefore, our algorithm could work even when the training group information is not available. We conduct experiments to verify its superiority over current group-oblivious algorithms.} The experimental results are shown in Table~\ref{tab:exp-group-oblivious} and we have the following observations.
\begin{itemize}
    \item (1) Benefiting from exploring the hard samples with importance reweighted mixup, the proposed \confmix achieves the best worst-case accuracy on Waterbirds, CelebA, and CivilComments datasets. For example, for the Waterbirds dataset, \confmix achieves worst-case accuracy of 90.0\%, while the second-best is 86.0\%. Moreover, on the CelebA dataset, the proposed method outperforms the second-best by 4.2\%.
    \item (2) As previously pointed out, ERM consistently outperforms other methods in terms of average accuracy. However, it typically comes with the lowest worst-case accuracy and highest average and worst-case accuracy gap. 
    \item (3) Compared to other methods, \confmix still shows competitive average accuracy. For example, on CelebA, \confmix achieves the average accuracy of 90.1\%, which outperforms all other IW/DRO methods.
    \item (4) The proposed method consistently achieves the lowest average and worst-case accuracy gap. Specifically, the proposed method achieves 3.0\%, 4.8\%, and 20.5\% accuracy gaps on Waterbirds, CelebA, and CivilComments datasets respectively, while the second-best is 7.6\%, 6.9\%, and 23.3\% respectively.
    \item (5) Benefiting from exploring difficult samples, on the Camelyon17 dataset, the proposed method achieves the best average accuracy. Specifically, \confmix achieves the best average accuracy of 75.1\% while the second is 70.3\%. This shows that the proposed method can improve the generalization of the model better than other reweighting algorithms.
\end{itemize}

\textbf{Q2 Effectiveness} (\uppercase\expandafter{\romannumeral2}). We further evaluate the Waterbirds and CelebA datasets without using the validation set group label information and the experimental results are shown in Table~\ref{tab:withoutvalidation}. Specifically, after each training epoch, we evaluate the performance of the current model on the validation set and save the model with the best average accuracy or best worst-case accuracy. Finally, we test the performance of the saved model on the test set. 
From the experimental results, we can observe that when the validation set group information is not used, the worst-case accuracy of our method drops a little while the average accuracy improves a little. Specifically, the average accuracy on the Waterbirds and CelebA datasets increased by 0.6\% and 0.3\%, respectively, while the worst-case accuracy decreased by 1.1\% and 0.7\%. It is worth noting that the gap between the average and worst-case accuracy will also increase slightly when the validation group label is not available.

\textbf{Q3 Qualitative analysis}. To intuitively investigate the rationality of the estimated uncertainty, we visualize the density of the uncertainty for different groups with kernel density estimation. As shown in Fig.~\ref{fig:quali}, the statistics of estimated uncertainty are basically correlated to the training sample size of each group. For example, on Waterbirds and CelebA, the average uncertainties of minority groups are much higher, while those of majority groups are much lower. 
\subsection{Ablation study}
In this section, we conduct the ablation study to show the proposed method could achieve better performance than mixup-based and importance weighting. The experimental results are reported in Table~\ref{tab:ablation}. From the results, we have the following observations.
\begin{itemize}
    \item (1) The proposed method could slightly improve the model's performance against subpopulation shifts when the group information is available during training. 
    \item (2) Due to the underutilization of the minority subpopulation‘s samples, mixup cannot improve the model's robustness to subpopulation shifts significantly, which also has been both experimentally and theoretically illustrated in previous work \cite{yao2022improving}.  
    \item (3) Compared with the mixup, the proposed method could achieve better performance in terms of both average and worst-case accuracy on all the datasets. For example, in the group-oblivious setting, the proposed method outperforms mixup by 33.8\%, 38.9\%, and 2.9\% in terms of worst-case accuracy on Waterbirds, CelebA, and CivilComments datasets. Meanwhile, in the group-aware setting, the proposed method also outperforms the in-group mixup.
    \item (4) By equipping mixup with importance weighting, the proposed method significantly improves the model robustness against subpopulation shifts. Specifically, which is also consistent with our conclusions stated in the Theorem~\ref{thm:generalization}, i.e., the proposed method has a tighter generalization bound than the importance weighting method.
\end{itemize}

\section{Conclusion}
In this paper, we aim to propose a method to improve the robustness against subpopulation shifts, which is critical for fair machine learning. To this end, we proposed a simple yet effective approach called \Umix to integrate importance weights into the well-known mixup strategy so that \Umix can mitigate the overfitting during training thus improving prior IW methods. We provide a strong theoretical intuition to show why the proposed method can improve performance.
To obtain reliable importance weights, we propose two different strategies to adapt to the group-obvious and group-aware settings of subpopulation shift. Specifically, in the group-aware setting, we propose a method based on subpopulation size to set the weight of samples. Moreover, in the group-oblivious setting, we propose methods to estimate the uncertainty of training samples and set weights accordingly. We also provide insightful theoretical analysis to show why the proposed method works better than vanilla mixup and CutMix from the Rademacher complexity perspective. Specifically, \Umix shows the theoretical advantage that the learned model comes with a subpopulation-heterogeneity dependent generalization bound.
Extensive experiments are conducted and the results show that the proposed method consistently outperforms previous approaches on commonly-used benchmarks.
In the future, we will think about how to employ \Umix in large-scale pre-trained models to improve the fairness of the large-scale model.
\bibliographystyle{IEEEtran}
\bibliography{main}

\ifCLASSOPTIONcaptionsoff
  \newpage
\fi





\newpage
\onecolumn
\appendices
\section{More details about uncertainty-based importance weights}
\label{sec:expdetail}
In this section, we present More details about uncertainty-based importance weights. Specifically, we describe uncertainty quantification results on simulated dataset in Sec.~\ref{sec:toyexp}, training accuracy of different subpopulations throughout training process in Sec.~\ref{sec:trainacc} and Justification for choosing historical-based uncertainty score in Sec.~\ref{sec:justification}.
\subsection{Uncertainty quantification results on simulated dataset}
\label{sec:toyexp}
We conduct a toy experiment to show the uncertainty quantification could work well on the dataset with subpopulation shift. Specifically, we construct a four moons dataset (i.e., a dataset with four subpopulations) as shown in Fig.~\ref{fig:toyexp}(a). We compare our approximation (i.e., Eq.~\ref{eq:approximation}) with the following ensemble-based approximation:
\begin{equation}
u_i \approx \frac{1}{T}\sum_{t=1}^{T}\kappa(y_i, \hat{f}_{\theta_t}(x_i))p(\theta_t;\mathcal{D})d\theta.
\end{equation}
Specifically, we train  $T$ models and then ensemble them. The quantification results are shown in Fig.~\ref{fig:toyexp}(b)(c). We can observe that (1) the proposed historical-based uncertainty quantification method could work well on the simulated dataset; (2) compared with the ensemble-based method, the proposed method could better characterize the subpopulation shift.

\begin{figure*}[!htbp]
\centering
\subfigure[Simulated dataset]{
\begin{minipage}[t]{0.3\linewidth}
\centering
\includegraphics[width=1\linewidth,height=0.75\linewidth]{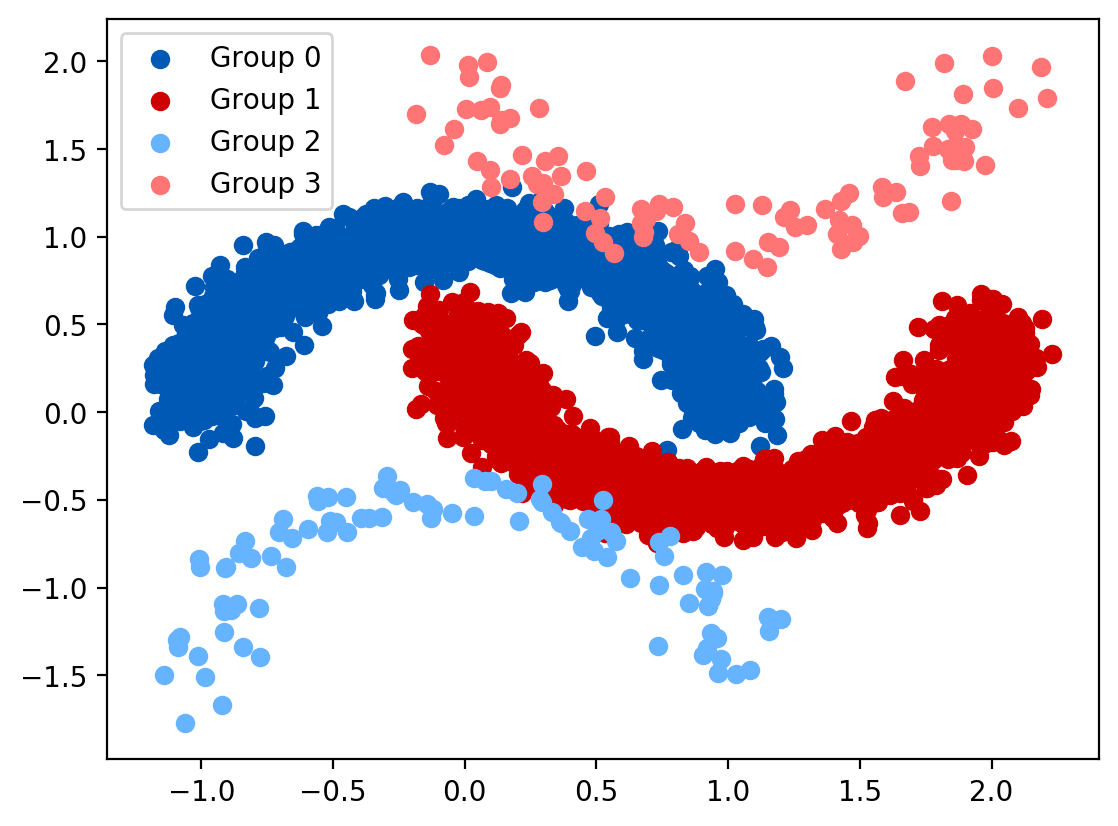}
\centering
\end{minipage}
}
\subfigure[Ours]{
\begin{minipage}[t]{0.3\linewidth}
\centering
\includegraphics[width=1\linewidth,height=0.75\linewidth]{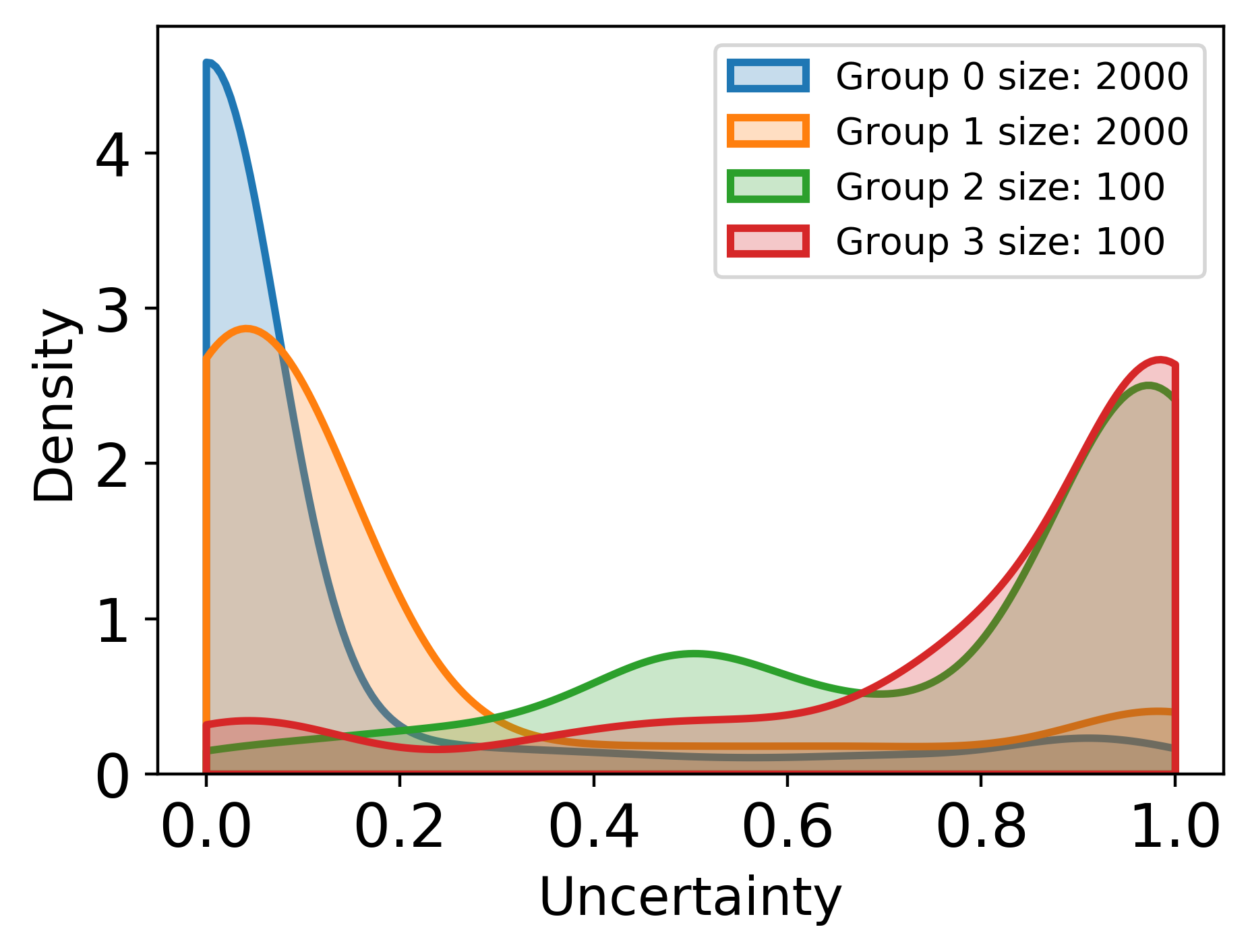}
\centering
\end{minipage}}
\subfigure[Ensemble]{
\begin{minipage}[t]{0.3\linewidth}
\centering
\includegraphics[width=1\linewidth,height=0.75\linewidth]{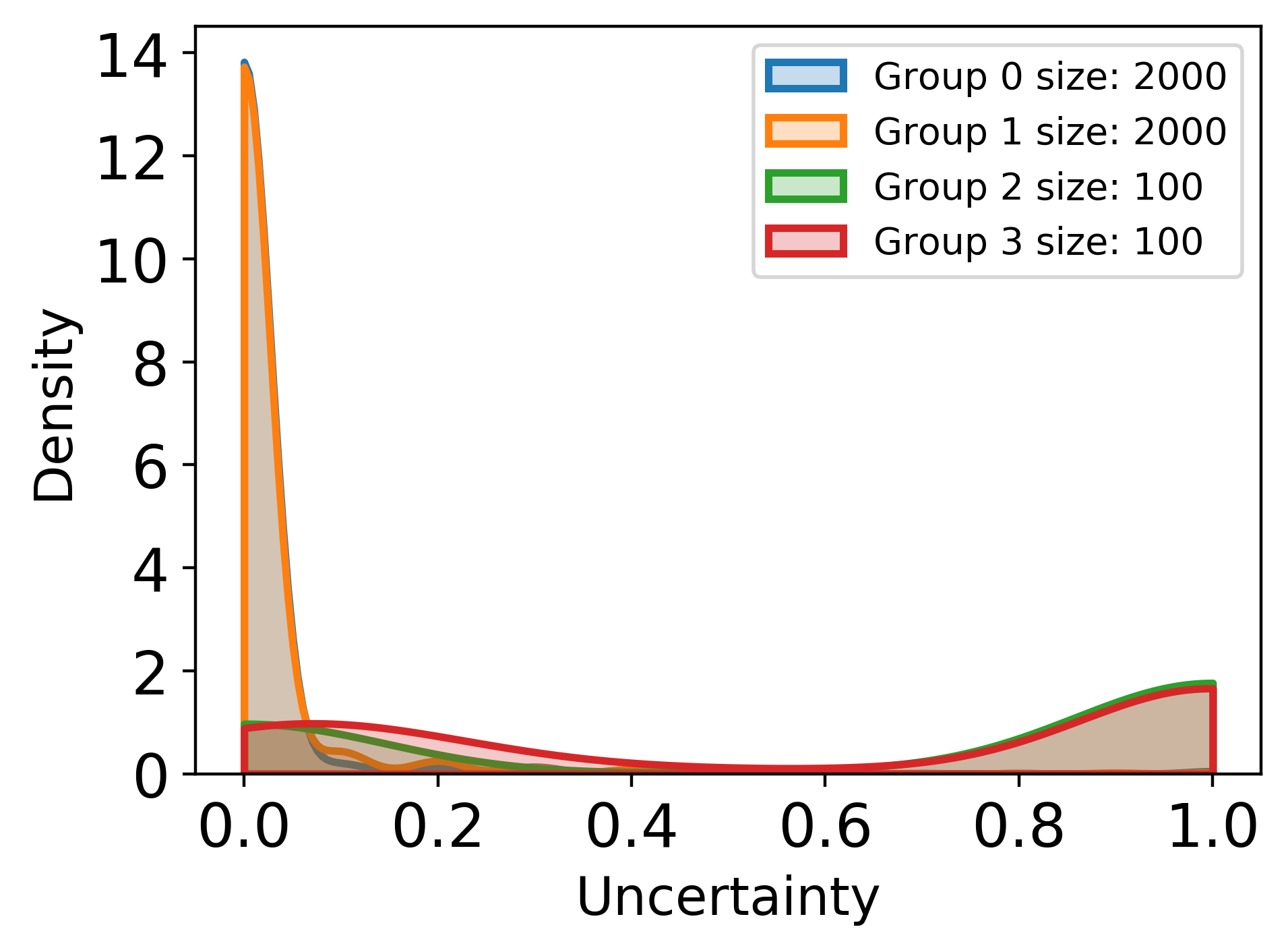}
\end{minipage}}
\caption{(a) Simulated dataset with four different subpopulations. In the four subpopulations, Group 0 and Group 2 have the
same label and groups 1 and 3 have the same labels. (b)(c) Visualization of the obtained uncertainty with kernel density estimation on simulated dataset, where group size refers to the sample number of the group. \label{fig:toyexp}}
\end{figure*}

\subsection{Training accuracy throughout training}
\label{sec:trainacc}
We present how the training accuracy change throughout training in Fig.~\ref{fig:trainingacc} on the CelebA and Waterbirds datasets to empirically show why the proposed estimation approach could work. From the experimental results, we observe that during training, easy groups with sufficient samples can be fitted well, and vice versa. For example, on the CelebA dataset, Group 0 and Group 1 with about 72K and 67K training samples quickly achieved over 95\% accuracy. The accuracy rate on Group 2, which has about 23K training samples, increased more slowly and finally reached around 84\%. The accuracy on Group 3, which has only about 1K training samples, is the lowest. Meanwhile, On the Waterbirds dataset, the samples of hard-to-classify group (e.g., Group 1) are also more likely to be forgotten by the neural networks.

\begin{figure*}[!htbp]
\centering
\subfigure[CelebA]{
\begin{minipage}[t]{0.48\linewidth}
\centering
\includegraphics[width=1\linewidth,height=0.6\linewidth]{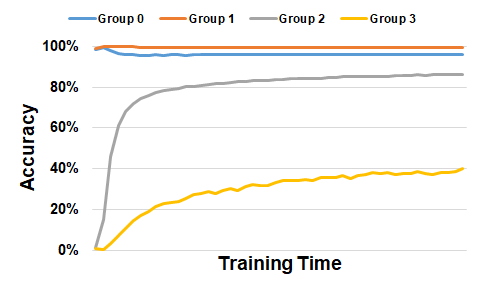}
\centering
\end{minipage}}
\subfigure[Waterbirds]{
\begin{minipage}[t]{0.48\linewidth}
\centering
\includegraphics[width=1\linewidth,height=0.6\linewidth]{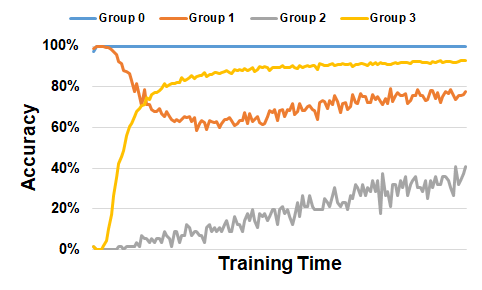}
\end{minipage}}
\caption{Visualization of the changing of training accuracy on different groups of  CelebA and Waterbirds datasets. \label{fig:trainingacc}}
\end{figure*}

\subsection{Justification for choosing historical-based uncertainty score}
\label{sec:justification}
We employ the information from the historical training trajectory to approximate the sampling process because it is simple and effective in practice. Empirically, in contrast to other typical uncertainty quantification methods such as Bayesian learning or model ensemble \cite{gal2016dropout,lakshminarayanan2017simple}, our method can significantly reduce the computational and memory-storage cost by employing the information from the historical training trajectory, since Bayesian learning or model ensemble needs to sample/save multiple DNN models and performs inference computations on them. Meanwhile, our method has achieved quite promising final accuracy in contrast to other methods. In summary, we choose an uncertainty score that can achieve satisfactory performance while being more memory and computationally efficient.

\section{Proof}
In this appendix, we prove the Theorem~5.1 in Section~5 using the template from the proof in \cite{park2022a} and incoporate our \IWmix.
We define a training dataset as $D=\{z_i = (x_i, y_i)\}_{i=1}^m$, randomly sampled from a distribution $\mathcal{P}_z$. Then for randomly selected two data samples, $z_i$ and $z_j$, $\tilde{z}_{i,j}^{\IWmix}$, is synthesized as follows
\begin{align*}
    \tilde{z}_{i,j}^{\IWmix}(\lambda) = (\tilde{x}_{i,j}^{\IWmix}(\lambda), \tilde{y}_{i,j}^{\IWmix}(\lambda)),
\end{align*}
where 
\begin{equation}
\begin{aligned}
\label{eq:IWmixup}
    \tilde{x}_{i,j}^{\IWmix}(\lambda) &= M(\lambda)\odot x_i + (1-M(\lambda))\odot x_j,\\
    \tilde{y}_{i,j}^{\IWmix}(\lambda) &= \lambda y_i + (1-\lambda)y_j,
\end{aligned}
\end{equation}
and $\lambda$ is the ratio parameter between samples drawn from $\mathcal{D}_{\lambda}$. $\odot$ means a component-wise multiplication in vector or matrix. $M(\lambda)$ is a random variable conditioned on $\lambda$ that indicates how we mix the input, e.g., by linear interpolation~\cite{zhang2018mixup} or by a pixel mask~\cite{yun2019cutmix}.

Under this template, two most popular Mixup methods, Mixup~\cite{zhang2018mixup} and CutMix~\cite{yun2019cutmix}, for $i$-th and $j$-th samples with $\lambda$ drawn from $\mathcal{D}_{\lambda}$, can be rewriten as follow
\begin{equation}
\begin{aligned}
& \text { Mixup: } \quad \tilde{z}_{i, j}^{(\text {mixup })}(\lambda)=\left(\tilde{x}_{i, j}^{(\text {mixup })}(\lambda), \tilde{y}_{i, j}^{(\text {mixup })}(\lambda)\right) \\
\text{ where } &\tilde{x}_{i, j}^{\text {(mixup })}(\lambda)=\lambda x_i+(1-\lambda) x_j \quad \text { and } \quad \tilde{y}_{i, j}^{(\text {mixup })}(\lambda)=\lambda y_i+(1-\lambda) y_j.
\end{aligned}
\end{equation}

\begin{equation}
\begin{aligned}
& \text { CutMix: } \quad \tilde{z}_{i, j}^{\text {(cutmix)}}(\lambda)=\left(\tilde{x}_{i, j}^{\text {(cutmix })}(M), \tilde{y}_{i, j}^{\text{(cutmix)}}(\lambda)\right) \\
\text { where } & \tilde{x}_{i, j}^{\text {(cutmix)}}\left({M}, 1-{M}\right)={M} \odot x_i+\left(1-{M}\right) \odot x_j \quad \text { and } \quad \tilde{y}_{i, j}^{\text {(cutmix) }}(\lambda)=\lambda y_i+(1-\lambda) y_j, \\
&
\end{aligned}
\end{equation}
i.e., the Mixup is a special case of Equation~\ref{eq:IWmixup} by putting $M(\lambda) = \lambda \overrightarrow{1}$ and the CutMixup is also a special case where ${M}$ is a binary mask that indicates the location of the cropped box region with a relative area.

We consider the following optimization objective, which is the expected version of our weighted mixup loss (Equation~\ref{eq:weighted-loss}).
\begin{align*}
    L_n^{\IWmix}(\theta, S) = \frac{1}{n^2} \sum^{n}_{i,j=1} \E_{\lambda \sim D_{\lambda}} \E_{M} [l(\theta, \tilde{z}_{i,j}(\lambda))],
\end{align*}
where the loss function we consider is $l(\theta, x,y) = h(f_{\theta}(x)) - yf_{\theta}(x)$ and $h(\cdot)$ and $f_{\theta}(\cdot)$ for all $\theta\in \Theta$ are twice differentiable. 
We compare it with the standard weighted loss function
\begin{align*}
    L_n(\theta, S) = \frac{1}{n} \sum_{i=1}^n w_i [h(f_\theta(x_i)) - y_i f_\theta(x_i)].
\end{align*}

\begin{lemma}
\label{lm:mixup-corresponding}
The weighted mixup loss can be rewritten as 
\begin{align*}
    L_{n}^{\IWmix}(\theta, S)=L_{n}(\theta, S)+\sum_{i=1}^{3} \mathcal{R}_{i}(\theta, S)+\mathbb{E}_{\lambda \sim \tilde{\mathcal{D}}_{\lambda}} \mathbb{E}_{M}\left[(1-M)^{\top}\varphi(1-M)(1-M)\right],
\end{align*}
where 
$\tilde{\mathcal{D}}_{\lambda}$ is a uniform mixture of two Beta distributions, i.e., $\frac{\alpha}{\alpha+\beta}Beta(\alpha+1,\beta) + \frac{\beta}{\alpha+\beta}Beta(\beta+1,\alpha)$ and $\psi(\cdot)$ is some function with $\lim_{a\rightarrow 0}\psi(a)=0$. Moreover, 
\begin{equation}
\begin{aligned}
& \mathcal{R}_1(\theta, S)=\frac{1}{n} \sum_{i=1}^n \left(y_i-h^{\prime}\left(f_\theta\left(x_i\right)\right)\right)\left(\nabla f_\theta\left(x_i\right)^{\top} x_i\right) \mathbb{E}_{\lambda \sim \tilde{D}_\lambda}(1-\lambda), \\
& \mathcal{R}_2(\theta, S)=\frac{1}{2 n} \sum_{i=1}^n h^{\prime \prime}\left(f_\theta\left(x_i\right)\right) \mathbb{E}_{\lambda \sim \tilde{D}_\lambda} \mathbf{G}\left(\mathcal{D}_X, x_i, f, M\right), \\
& \mathcal{R}_3(\theta, S)=\frac{1}{2 n} \sum_{i=1}^n \left(h^{\prime}\left(f_\theta\left(x_i\right)\right)-y_i\right) \mathbb{E}_{\lambda \sim \tilde{D}_\lambda} \mathbf{H}\left(\mathcal{D}_X, x_i, f, M\right),
\end{aligned}
\end{equation}
and 
\begin{equation}
\begin{aligned}
& \mathbf{G}\left(\mathcal{D}_X, x_i, f, M\right)=\mathbb{E}_M(1-M)^{\top} \mathbb{E}_{r_x \sim \mathcal{D}_X}\left(\nabla f\left(x_i\right) \odot\left(r_x-x_i\right)\left(\nabla f\left(x_i\right) \odot\left(r_x-x_i\right)\right)^{\top}\right)(1-M) \\
& =\sum_{j, k \in \text { coord }} a_{j k} \partial_j f_\theta\left(x_i\right) \partial_k f_\theta\left(x_i\right)\left(\mathbb{E}_{r_x \sim \mathcal{D}_X}\left[r_{x j} r_{x k}\right]+x_{i j} x_{i k}\right), \\
& \mathbf{H}\left(\mathcal{D}_X, x_i, f, M\right)=\mathbb{E}_{r_x \sim \mathcal{D}_X} \mathbb{E}_M(1-M)^{\top}\left(\nabla^2 f_\theta\left(x_i\right) \odot\left(\left(r_x-x_i\right)\left(r_x-x_i\right)^{\top}\right)\right)(1-M) \\
& =\sum_{j, k \in \text { coord }} a_{j k}\left(\mathbb{E}_{r_x \sim \mathcal{D}_X}\left[r_{x j} r_{x k} \partial_{j k}^2 f_\theta\left(x_i\right)\right]+x_{i j} x_{i k} \partial_{j k}^2 f_\theta\left(x_i\right)\right).
\end{aligned}
\end{equation}
\end{lemma}

\begin{proof}
\label{subsec:proof-lemma-mixup}
Due to the assumption of the loss function, we can rewrite the empirical loss for the non-augmented population as
$$
L_n(\theta)=\frac{1}{n} \sum_{i=1}^n w_i l\left(\theta, z_i\right)=\frac{1}{n} \sum_{i=1}^n w_i \left[h\left(f_\theta\left(x_i\right)\right)-y_i f_\theta\left(x_i\right)\right] .
$$
Similarly, we can rewrite the \IWmix loss as
$$
\begin{aligned}
L_n^{\IWmix}(\theta) & =\frac{1}{n^2} \sum_{i, j=1}^n \mathbb{E}_{\lambda \sim \operatorname{Beta}(\alpha, \beta)} \mathbb{E}_M l\left(\theta, \tilde{z}_{i, j}^{\IWmix}(M)\right) \\
& =\frac{1}{n^2} \sum_{i, j=1}^n \mathbb{E}_{\lambda \sim \operatorname{Beta}(\alpha, \beta)} \mathbb{E}_M\left[h\left(f_\theta\left(\tilde{x}_{i, j}^{\IWmix}(M)\right)\right)-\tilde{y}_{i, j}^{\IWmix}(M) f_\theta\left(\tilde{x}_{i, j}^{\IWmix}(M)\right)\right] .
\end{aligned}
$$

Putting the definition of $\tilde{z}_{i, j}^{\IWmix}(M)$ to the equation above, we have
$$
\begin{aligned}
L_n^{\IWmix}(\theta) & =\frac{1}{n^2} \sum_{i, j=1}^n\left(\mathbb { E } _ { \lambda \sim \operatorname { B e t a } ( \alpha , \beta ) } \mathbb { E } _ { M } \left(\lambda w_i\left(h\left(f_\theta\left(\tilde{x}_{i, j}^{\IWmix}(M)\right)\right)-y_i f_\theta\left(\tilde{x}_{i, j}^{\IWmix}(M)\right)\right)\right.\right. \\
& \left.\left.+(1-\lambda)w_j\left(h\left(f_\theta\left(\tilde{x}_{i, j}^{\IWmix}(M)\right)\right)-y_j f_\theta\left(\tilde{x}_{i, j}^{\IWmix}(M)\right)\right)\right)\right) \\
& =\frac{1}{n^2} \sum_{i, j=1}^n\left(\mathbb { E } _ { \lambda \sim \operatorname { B e t a } ( \alpha , \beta ) } \mathbb { E } _ { B \sim \operatorname { B i n } ( \lambda ) } \mathbb { E } _ { M } \left(Bw_i\left(h\left(f_\theta\left(\tilde{x}_{i, j}^{\IWmix}(M)\right)\right)-y_i f_\theta\left(\tilde{x}_{i, j}^{\IWmix}(M)\right)\right)\right.\right. \\
& \left.\left.+(1-B)w_j\left(h\left(f_\theta\left(\tilde{x}_{i, j}^{\IWmix}(M)\right)\right)-y_j f_\theta\left(\tilde{x}_{i, j}^{\IWmix}(M)\right)\right)\right)\right)
\end{aligned}
$$
Note that $\lambda \sim \operatorname{Beta}(\alpha, \beta)$ and $B \mid \lambda \sim \operatorname{Bin}(\lambda)$. By conjugacy, we can write the joint distribution of $(\lambda, B)$ as
$$
B \sim \operatorname{Bin}\left(\frac{\alpha}{\alpha+\beta}\right), \quad \lambda \mid B \sim \operatorname{Beta}(\alpha+B, \beta+1-B) .
$$

Therefore, we have
\begin{align}
L_n^{\IWmix}(\theta) & =\frac{1}{n^2} \sum_{i, j=1}^n\left(\mathbb{E}_{\lambda \sim \operatorname{Beta}(\alpha+1, \beta)} \mathbb{E}_M \frac{\alpha}{\alpha+\beta}w_i\left(h\left(f_\theta\left(\tilde{x}_{i, j}^{\IWmix}(M)\right)\right)-y_i f_\theta\left(\tilde{x}_{i, j}^{\IWmix}(M)\right)\right)\right. \notag \\
& \left.+\mathbb{E}_{\lambda \sim \operatorname{Beta}(\alpha, \beta+1)} \mathbb{E}_M \frac{\beta}{\alpha+\beta}w_j\left(h\left(f_\theta\left(\tilde{x}_{i, j}^{\IWmix}(M)\right)\right)-y_j f_\theta\left(\tilde{x}_{i, j}^{\IWmix}(M)\right)\right)\right) \notag\\
& =\frac{1}{n} \sum_{i=1}^n w_i \mathbb{E}_{\lambda \sim \tilde{\mathcal{D}}(\lambda)} \mathbb{E}_{r_x \sim \mathcal{D}_x} \mathbb{E}_M\left[h\left(f_\theta\left(w_i M \odot x_i+(1-M) \odot r_x\right)\right)-y_i f_\theta\left(w_i M \odot x_i+(1-M) \odot r_x\right)\right] \label{eq:loss_decompose} \\
& =\frac{1}{n} \sum_{i=1}^n w_i\mathbb{E}_{\lambda \sim \tilde{\mathcal{D}(\lambda)}} \mathbb{E}_{r_x \sim \mathcal{D}_x} \mathbb{E}_M l\left(\theta, \hat{z}_i\right) \label{eq:mixup_loss},
\end{align}
where $\mathcal{D}_x$ is the empirical distribution induced by training samples and their corresponding weights, and $\hat{z}_i=\left(w_i M \odot x_i+(1-M) \odot r_x, y_i\right)$.

Let $N=1-M$. By defining $\phi_i(N)=h\left(f_\theta\left(w_i x_i+N \odot\left(r_x-x_i\right)\right)\right)-y_i f_\theta\left(w_i x_i+N \odot\left(r_x-x_i\right)\right)$ and applying Taylor expansion, we have
\begin{equation}
\label{eq:taylor}
\phi_i(N)=\phi_i(0)+\nabla_N \phi_i(0)^{\top} N+\frac{1}{2} N^{\top} \nabla_N^2 \phi_i(0) N+N^{\top} \varphi(N) N,
\end{equation}
where $\lim _{N \rightarrow 0} \varphi(N)=0$. 
Firstly, we calculate $\phi_i(0)$ by
\begin{equation}
\label{eq:phi_0}
\phi_i(0)=h\left(f_\theta\left(w_i x_i\right)\right)-y_i f_\theta\left(w_i x_i\right) .
\end{equation}
Second, we calculate $\nabla_N \phi_i(0)$ by
$$
\frac{\partial \phi_i(N)}{\partial N_k}=\left(h^{\prime}\left(f_\theta\left(w_i x_i+N \odot\left(r_x-x_i\right)\right)\right)-y_i\right) \frac{\partial f_\theta}{\partial x_{i k}}\left(w_i x_i+N \odot\left(r_x-x_i\right)\right)\left(r_{x k}-x_{i k}\right),
$$
where we denote $N_k$ as the $k$ th element of $N, x_{i k}$ as the $k$ th element of $x_i$, and $r_{x k}$ as the $k$ th element of $r_x$. 
Therefore, we have
\begin{equation}
\label{eq:phi_1}
\begin{aligned}
\nabla_N \phi_i(0)^{\top} N & =\left(h^{\prime}\left(f_\theta\left(x_i\right)\right)-y_i\right) \sum_k\left(\frac{\partial f_\theta}{\partial x_{i k}}\left(w_i x_i\right)\left(r_{x k}-x_{i k}\right)\right) N_k \\
& =\left(h^{\prime}\left(f_\theta\left(x_i\right)\right)-y_i\right)\left(\nabla f \odot\left(r_x-x_i\right)\right) \cdot N.
\end{aligned}
\end{equation}

Finally, we calculate $\nabla_N^2 \varphi_i(\overrightarrow{0})^\top$ by
$$
\begin{aligned}
\frac{\partial^2 \phi_k(N)}{\partial N_k \partial N_j}= & \frac{\partial}{\partial N_j}\left(\left(h^{\prime}\left(f_\theta\left(w_i x_i+N \odot\left(r_x-x_i\right)\right)\right)-y_i\right) \frac{\partial f_\theta}{\partial x_{i k}}\left(w_i x_i+N \odot\left(r_x-x_i\right)\right)\left(r_{x k}-x_{i k}\right)\right) \\
= & h^{\prime \prime}\left(f_\theta\left(w_i x_i+N \odot\left(r_x-x_i\right)\right)\right) \\
& \times \frac{\partial f_\theta}{\partial x_{i k}}\left(w_i x_i+N \odot\left(r_x-x_i\right)\right)\left(r_{x k}-x_{i k}\right) \frac{\partial f_\theta}{\partial x_{i j}}\left(w_i x_i+N \odot\left(r_x-x_i\right)\right)\left(r_{x j}-x_{i j}\right) \\
+ & \left(h^{\prime}\left(f_\theta\left(w_i x_i+N \odot\left(r_x-x_i\right)\right)\right)-y_i\right) \\
& \times \frac{\partial^2 f_\theta}{\partial x_{i k} \partial x_{i j}}\left(w_i x_i+N \odot\left(r_x-x_i\right)\right)\left(r_{x k}-x_{i k}\right)\left(r_{x j}-x_{i j}\right) .
\end{aligned}
$$
Therefore, we have
\begin{equation}
\begin{aligned}
\frac{1}{2} N^{\top} \nabla_N^2 \phi_i(0) N & = \frac{1}{2} h^{\prime \prime}\left(f_\theta\left(x_i\right)\right) \sum_{k, j}\left(\frac{\partial f_\theta}{\partial x_{i k}}\left(x_i\right)\left(r_{x k}-x_{i k}\right) \frac{\partial f_\theta}{\partial x_{i j}}\left(x_i\right)\left(r_{x j}-x_{i j}\right) N_k N_j\right) \\
&\quad +\frac{1}{2}\left(h^{\prime}\left(f_\theta\left(x_i\right)\right)-y_i\right) \sum_{k, j} \frac{\partial^2 f_\theta}{\partial x_{i k} \partial x_{i j}}\left(x_i\right)\left(r_{x k}-x_{i k}\right)\left(r_{x j}-x_{i j}\right) N_k N_j \notag \\
&=\frac{1}{2} h^{\prime \prime}\left(f_\theta\left(x_i\right)\right) N^{\top}\left(\left(\nabla f \odot\left(r_x-x_i\right)\right)\left(\nabla f \odot\left(r_x-x_i\right)\right)^{\top}\right) N \\
&\quad +\frac{1}{2}\left(h^{\prime}\left(f_\theta\left(x_i\right)\right)-y_i\right) N^{\top}\left(\nabla^2 f_\theta\left(x_i\right) \odot\left(\left(r_x-x_i\right)\left(r_x-x_i\right)^{\top}\right)\right) N \label{eq:phi_2}.
\end{aligned}
\end{equation}

Applying \ref{eq:phi_0}-\ref{eq:phi_2} to \ref{eq:taylor},
\begin{equation}
\label{eq:phi_3}
\begin{aligned}
\phi_i(N) & =\left(h\left(f_\theta\left(x_i\right)\right)-y_i f_\theta\left(x_i\right)\right)+\left(h^{\prime}\left(f_\theta\left(x_i\right)\right)-y_i\right)\left(\nabla f \odot\left(r_x-x_i\right)\right) \cdot N \\
& +\frac{1}{2} h^{\prime \prime}\left(f_\theta\left(x_i\right)\right) N^{\top}\left(\left(\nabla f \odot\left(r_x-x_i\right)\right)\left(\nabla f \odot\left(r_x-x_i\right)\right)^{\top}\right) N \\
& +\frac{1}{2}\left(h^{\prime}\left(f_\theta\left(x_i\right)\right)-y_i\right) N^{\top}\left(\nabla^2 f_\theta\left(x_i\right) \odot\left(\left(r_x-x_i\right)\left(r_x-x_i\right)^{\top}\right)\right) N+N^{\top} \varphi(N) N
\end{aligned}
\end{equation}
Plugging \ref{eq:phi_3} to \ref{eq:loss_decompose}, we conclude
$$
\begin{aligned}
L_n^{\IWmix}(\theta) & =\frac{1}{n} \sum_{i=1}^n \mathbb{E}_{\lambda \sim \tilde{\mathcal{D}}(\lambda)} \mathbb{E}_{r_x \sim \mathcal{D}_x} \mathbb{E}_M \phi(1-M) \\
& =L_n(\theta)+\mathcal{R}_1(\theta)+\mathcal{R}_2(\theta)+\mathcal{R}_3(\theta)+\mathbb{E}_{\lambda \sim \tilde{\mathcal{D}}(\lambda)} \mathbb{E}_M\left[(1-M)^{\top} \varphi(1-M) M\right],
\end{aligned}
$$
where
$$
\begin{aligned}
& \mathcal{R}_1(\theta)=\frac{1}{n} \sum_{i=1}^n\left(h^{\prime}\left(f_\theta\left(x_i\right)\right)-y_i\right)\left(\nabla f_\theta\left(x_i\right) \odot \mathbb{E}_{r_x \sim \mathcal{D}_X}\left[r_x-x_i\right]\right) \mathbb{E}_{\lambda \sim \tilde{D}_\lambda} \mathbb{E}_M(1-M), \\
& \mathcal{R}_2(\theta)=\frac{1}{2 n} \sum_{i=1}^n h^{\prime \prime}\left(f_\theta\left(x_i\right)\right) \mathbb{E}_{\lambda \sim \tilde{D}_\lambda} \mathbb{E}_M(1-M)^{\top} \mathbb{E}_{r_x \sim \mathcal{D}_X}\left[\nabla f\left(x_i\right) \odot\left(r_x-x_i\right)\left(\nabla f\left(x_i\right) \odot\left(r_x-x_i\right)\right)^{\top}\right](1-M), \\
& \mathcal{R}_3(\theta)=\frac{1}{2 n} \sum_i^n \mathbb{E}_{\lambda \sim \tilde{D}_\lambda} \mathbb{E}_M(1-M)^{\top} \mathbb{E}_{r_x \sim \mathcal{D}_X}\left[\nabla^2 f_\theta\left(x_i\right) \odot\left(\left(r_x-x_i\right)\left(r_x-x_i\right)^{\top}\right)\right](1-M) .
\end{aligned}
$$
\end{proof}

\begin{lemma}
\label{lm:mixup-GLM-loss}
Consider the centralized dataset, i.e., $\frac{1}{n}\sum_{i=1}^n x_i=0$, we have 
\begin{align*}
    \E_{\lambda\sim \tilde{\mathcal{D}}_{\lambda}}[L_n^{mix}(\theta, \tilde{S})] \approx L_n(\theta, S) + \frac{C}{2n} [\sum_{i=1}^n w_iA^{\prime\prime}(x_i^{\top}\theta)]\theta^{\top}(\mathbb{E}(1-M)\widehat{\Sigma}_X(1-M)^{\top} + x_i\operatorname{Var}(M)x_i^{\top})\theta,
\end{align*}
where $\widehat{\Sigma}_X=\frac{1}{n}\sum_{i=1}^n w_i x_i x_i^{\top}$, $C$ is some constant, 
and the expectation is taken with respect to the randomness of $\lambda$ and $W$.
\end{lemma}

\label{subsec:proof-of-mixup-GLM}
\begin{proof}
For GLM, using Equation~\ref{eq:mixup_loss}, since the prediction of the GLM model is invariant to the scaling of the training data, we consider the dataset $\hat{D}=\left\{\hat{z}_i\right\}_{i=1}^n$ with $\hat{x_i}=1 \oslash \bar{M} \odot\left(w_i M \odot x_i+(1-M) \odot r_x\right)$ where $\bar{M}=\mathbb{E} M$. Then, the loss function is
$$
L_m^{\IWmix}=\frac{1}{n}  \mathbb{E}_\lambda \mathbb{E}_{r_x} \mathbb{E}_M \sum_{i=1}^n w_i l\left(\theta, \tilde{z}_i\right)=\frac{1}{n} \mathbb{E}_{\xi} \sum_{i=1}^n w_i\left(A\left(\hat{x}_i^{\top} \theta\right)-y_i \hat{x}_i^{\top} \theta\right),
$$
where $\xi$ denotes the randomness of $\lambda, r_x$, and $M$. By the second approximation of $A(\cdot)$, we can express $A\left(\hat{x}_i^{\top} \theta\right)$ as
$$
A\left(\hat{x}_i^{\top} \theta\right)=A\left(x_i^{\top} \theta\right)+A^{\prime}\left(x_i^{\top} \theta\right)\left(\hat{x_i}-x_i\right)^{\top} \theta+\frac{1}{2} A^{\prime \prime}\left(x_i^{\top} \theta\right) \theta^{\top}\left(\hat{x_i}-x_i\right)\left(\hat{x_i}-x_i\right)^{\top} \theta
$$
to approximate the loss function. Therefore, we have
$$
\begin{aligned}
\tilde{L}_m^{\IWmix} & =\frac{1}{n} \sum_{i=1}^n w_i A\left(x_i^{\top} \theta\right)+\frac{1}{n} \mathbb{E}_{\xi} \sum_{i=1}^n w_i\left(A^{\prime}\left(x_i^{\top} \theta\right)\left(\hat{x}_i-x_i\right)^{\top} \theta+\frac{1}{2} A^{\prime \prime}\left(x_i^{\top} \theta\right) \theta^{\top}\left(\hat{x}_i-x_i\right)\left(\hat{x}_i-x_i\right)^{\top} \theta\right) \\
& =\frac{1}{n} \sum_{i=1}^n w_i A\left(x_i^{\top} \theta\right)+\frac{1}{n} \sum_{i=1}^n w_i\left(\frac{1}{2} A^{\prime \prime}\left(x_i^{\top} \theta\right) \theta^{\top} \operatorname{Var}_{\xi}\left(\hat{x_i}\right) \theta\right),
\end{aligned}
$$
where $\tilde{L}_m^{\IWmix}$ denotes the approximate loss of $L_m^{\IWmix}$ since $\mathbb{E}_{\xi} r_x=0$ and $\mathbb{E}_{\xi} \hat{x}_i=x_i$. For calculating $\operatorname{Var}_{\xi}\left(\hat{x}_i\right)$, we use the law of total variance. We have
$$
\begin{aligned}
\operatorname{Var}_{\xi}\left(\tilde{x_i}\right) & =\left(\frac{1}{\bar{M}} \frac{1}{\bar{M}}^{\top}\right) \odot \operatorname{Var}_{\xi}\left(M \odot x_i+(1-M) \odot r_x\right) \\
& =\left(\frac{1}{\bar{M}} \frac{1}{\bar{M}}^{\top}\right) \odot\left(\mathbb { E } \left(\operatorname{Var}\left(M \odot x_i+(1-M) \odot r_x \mid \lambda, M\right)+\operatorname{Var}\left(\mathbb{E}\left(M \odot x_i+(1-M) \odot r_x \mid \lambda, M\right)\right)\right.\right. \\
& \left.=\left(\frac{1}{\bar{M}} \frac{1}{\bar{M}}^{\top}\right) \odot\left(\mathbb{E}(1-M) \hat{\Sigma}_X(1-M)^{\top}+x_i \operatorname{Var}(M) x_i^{\top}\right)\right) \\
& =\frac{1}{\bar{\lambda}^2}\left(\mathbb{E}(1-M) \hat{\Sigma}_X(1-M)^{\top}+x_i \operatorname{Var}(M) x_i^{\top}\right)
\end{aligned}
$$
where $\hat{\Sigma}_X=\frac{1}{n} \sum_{i=1}^n w_i x_i x_i^{\top}$ with some notational ambiguity that $\frac{1}{n}:=\overrightarrow{1} \oslash \bar{M}$. In our setting $\bar{M}=\bar{\lambda} \overrightarrow{1}$ where $\bar{\lambda}=\mathbb{E}_{\lambda \sim \tilde{\mathcal{D}}_\lambda}[\lambda]$. Now we think the related dual problem:
$$
\begin{aligned}
\mathcal{W}_\gamma=\left\{x \rightarrow \theta^{\top} x, \text { such that } \theta : \left(\mathbb{E}_x A^{\prime \prime}\left(\theta^{\top} x\right)\right) \cdot\left(\theta^{\top}\left(\mathbb{E}(1-M) \Sigma_X(1-M)^{\top}\right) \theta+\theta^{\top}\left(\left(x \operatorname{Var}(M) x^{\top}\right)\right) \theta\right) \leq \gamma\right\} .
\end{aligned}
$$
\end{proof}

\begin{lemma}
\label{lm:rademacher}
Assume that the distribution of $x_i$ is $\rho$-retentive, i.e., satisfies the Assumption~\ref{as:rho-retentive}. 
Then the Rademacher complexity of $\mathcal{W}_r$ satisfies 
\begin{align*}
    Rad(\mathcal{W}_r) \le \frac{1}{\sqrt{n}}(\gamma / \rho)^{1 / 4}\left(\sqrt{\operatorname{tr}\left(\left(\Sigma_X^{(M)}\right)^{\dagger} \Sigma_X\right)}+\operatorname{rank}\left(\Sigma_X\right)\right).
\end{align*}
\end{lemma}

\begin{proof}
Define $\Sigma_X^{(M)} = \mathbb{E}[(1-M)\Sigma_X (1-M)^{\top}]$.
The proof is mainly based on~\cite{zhang2021how}.
By definition, given $n$ i.i.d. Rademacher rv. $\xi_{1}, \ldots, \xi_{n}$, the empirical Rademacher complexity is
\begin{align*}
\operatorname{Rad}\left(\mathcal{W}_{\gamma}, S\right)&=\mathbb{E}_{\xi} \sup_{a(\theta) \cdot \theta^{\top} (\Sigma_{X}^{(M)} + x\operatorname{Var}(M)x^{\top}) \theta \leq \gamma} \frac{1}{n} \sum_{i=1}^{n} \xi_{i} \theta^{\top} x_{i}\\
&\le \mathbb{E}_{\xi} \sup_{(\mathbb{E}_xA^{\prime \prime}(\theta^{\top} x)) \cdot \theta^{\top} \Sigma_{X}^{(M)}\theta \leq \gamma} \frac{1}{n} \sum_{i=1}^{n} \xi_{i} \theta^{\top} x_{i}\\
&\le \mathbb{E}_{\xi} \sup_{\mathbb{E}_xA^{\prime \prime}(\theta^{\top} x)) \cdot \theta^{\top} \Sigma_{X}^{(M)}\theta \leq \gamma} \frac{1}{n} \sum_{i=1}^{n} \xi_{i} \theta^{\top} x_{i}\\
&\le \mathbb{E}_{\xi} \sup_{(\theta^{\top} \Sigma_X \theta) \cdot \theta^{\top} \Sigma_{X}^{(M)}\theta \leq \gamma} \frac{1}{n} \sum_{i=1}^{n} \xi_{i} \theta^{\top} x_{i}\\
&\leq \mathbb{E}_{\xi_i}\left(\sup _{\theta \top \Sigma_X \theta \leq \sqrt{\gamma / \rho}} \frac{1}{n} \sum_{i=1}^n \xi_i \theta^{\top} x_i+\mathbb{E}_{\xi_i} \sup _{\theta \top \Sigma_X^{(M)} \theta \leq \sqrt{\gamma / \rho}} \frac{1}{n} \sum_{i=1}^n \xi_i \theta^{\top} x_i\right),
\end{align*}
where the third inequality is from Assumption~\ref{as:rho-retentive}.

For the first part of the RHS, define $\tilde{x}_i=\Sigma_X^{\dagger / 2} x_i$ and $v=\Sigma_X^{1 / 2} \theta$. Then, we have
$$
\begin{aligned}
& \mathbb{E}_{\xi_i} \sup _{\theta \top \Sigma_X \theta \leq \sqrt{\gamma / \rho}} \frac{1}{n} \sum_{i=1}^n \xi_i \theta^{\top} x_i=\mathbb{E}_{\xi_i} \sup _{\|v\|^2 \leq \sqrt{\gamma / \rho}} \frac{1}{n} \sum_{i=1}^n \xi_i v^{\top} \tilde{x}_i \\
& \leq \frac{1}{n}(\gamma / \rho)^{1 / 4} \mathbb{E}_{\xi_i}\left\|\sum_{i=1}^n \xi_i \tilde{x}_i\right\| \leq \frac{1}{n}(\gamma / \rho)^{1 / 4} \sqrt{\mathbb{E}_{\xi_i}\left\|\sum_{i=1}^n \xi_i \tilde{x}_i\right\|^2} \\
& =\frac{1}{n}(\gamma / \rho)^{1 / 4} \sqrt{\sum_{i=1}^n \tilde{x}_i^{\top} \tilde{x}_i} . \\
&
\end{aligned}
$$
Similarly, by defining $\check{x}_i=\left(\Sigma_X^{(M)}\right)^{\dagger / 2} x_i$ and $v=\left(\Sigma_X^{(M)}\right)^{1 / 2} \theta$,
$$
\begin{aligned}
\mathbb{E}_{\xi_i} \sup _{\theta \top \Sigma_X^{(M)} \theta \leq \sqrt{\gamma / \rho}} \frac{1}{n} \sum_{i=1}^n \xi_i \theta^{\top} x_i & =\mathbb{E}_{\xi_i} \sup _{\|v\|^2 \leq \sqrt{\gamma / \rho}} \frac{1}{n} \sum_{i=1}^n \xi_i v^{\top} \check{x}_i \\
& \leq \frac{1}{n}(\gamma / \rho)^{1 / 4} \mathbb{E}_{\xi_i}\left\|\sum_{i=1}^n \xi_i \check{x}_i\right\| \leq \frac{1}{n}(\gamma / \rho)^{1 / 4} \sqrt{\mathbb{E}_{\xi_i}\left\|\sum_{i=1}^n \xi_i \check{x}_i\right\|^2} \\
& =\frac{1}{n}(\gamma / \rho)^{1 / 4} \sqrt{\sum_{i=1}^n \check{x}_i^{\top} \check{x}_i} .
\end{aligned}
$$
Therefore,
$$
\begin{aligned}
\operatorname{Rad}\left(\mathcal{W}_\gamma\right) & =\mathbb{E}\left[\operatorname{Rad}\left(\mathcal{W}_\gamma, n\right)\right] \leq \frac{1}{n}(\gamma / \rho)^{1 / 4}\left(\sqrt{\sum_{i=1}^n \mathbb{E}_x \tilde{x}_i^{\top} \tilde{x}_i}+\sqrt{\sum_{i=1}^n \mathbb{E}_x \check{x}_i^{\top} \check{x}_i}\right) \\
& \leq \frac{1}{\sqrt{n}}(\gamma / \rho)^{1 / 4}\left(\sqrt{\operatorname{tr}\left(\left(\Sigma_X^{(M)}\right)^{\dagger} \Sigma_X\right)}+\operatorname{rank}\left(\Sigma_X\right)\right).
\end{aligned}
$$

Based on this bound on Rademacher complexity, Theorem ~\ref{thm:generalization} can be proved by directly applying the Theorem 8 from \cite{bartlett2002rademacher}.
\end{proof}


\end{document}